\documentclass[11pt]{article}

\usepackage{amsmath}
\usepackage{amsthm}
\usepackage{hyperref}
\usepackage{times}
\usepackage{natbib}
\usepackage{graphicx} 
\usepackage{epstopdf}
\usepackage{subcaption}
\usepackage{amsfonts}
\usepackage{amssymb}
\usepackage{bm}
\usepackage{bbm}
\usepackage{mathtools}
\usepackage{algorithm}
\usepackage{etoolbox}
\usepackage{algpseudocode}
\usepackage{hhline} 
\usepackage{epstopdf}
\usepackage{float}
\usepackage{ulem} 
\usepackage{dsfont}

\usepackage[margin=35mm]{geometry}

\algnewcommand\algorithmicinput{\textbf{Input:}}
\algnewcommand\Input{\item[\algorithmicinput]}
\algnewcommand\algorithmicoutput{\textbf{Output:}}
\algnewcommand\Output{\item[\algorithmicoutput]}
\algnewcommand\algorithmicbreak{\textbf{break}}
\algnewcommand\Break{\item[\algorithmicbreak]}

\makeatletter
\newenvironment{subtheorem}[1]{%
	\def\subtheoremcounter{#1}%
	\refstepcounter{#1}%
	\protected@edef\theparentnumber{\csname the#1\endcsname}%
	\setcounter{parentnumber}{\value{#1}}%
	\setcounter{#1}{0}%
	\expandafter\def\csname the#1\endcsname{\theparentnumber.\Alph{#1}}%
	\ignorespaces
}{%
	\setcounter{\subtheoremcounter}{\value{parentnumber}}%
	\ignorespacesafterend
}
\makeatother
\newcounter{parentnumber}

\newtheorem{thm}{Theorem}
\newtheorem{remark}{Remark}
\newtheorem{lemma}{Lemma}
\newtheorem{corollary}{Corollary}
\newtheorem{definition}{Definition}

\usepackage[draft]{todonotes}

\renewcommand{\bm}{\mathbf}

\newcommand{\sig}{\bm s}
\newcommand{\errv}{\boldsymbol \xi}
\newcommand{\err}{ \xi}

\newcommand{\indset}{S}

\newcommand{\sign}{\operatorname{sign}}
\newcommand{\Trunc}{\operatorname{Trunc}}
\newcommand{\Approx}{\operatorname{Approx}}
\newcommand{\E}{\mathbb{E}}

\newcommand{\Nor}{\mathcal N}

\renewcommand{\u}{\bm u}
\newcommand{\V}{\bm V}
\renewcommand{\v}{\bm v}
\newcommand{\muv}{\boldsymbol \mu}
\newcommand{\mumin}{\mu_{\min}}
\newcommand{\mumax}{\mu_{\max}}
\newcommand{\R}{\mathbb{R}}

\newcommand{\x}{\bm{x}}
\newcommand{\y}{\bm{y}}
\newcommand{\Z}{Z}
\newcommand{\s}{\sig}

\newcommand{\w}{\bm{w}}
\newcommand{\tc}{t_c}
\newcommand{\tm}{t_m}
\newcommand{\htm}{\hat{t}_m}
\newcommand{\Mc}{M_c}
\newcommand{\Meff}{\Mc}

\newcommand{\Od}{\mathcal{O}_{\mbox{\tiny diag}}}
\newcommand{\SdK}{\mathcal{S}_{d,K}}
\newcommand{\Obb}{\mathcal{O}_{\mbox{\tiny BB}}}
\newcommand{\Osim}{\mathcal{O}_{\mbox{\tiny sim}}}
\newcommand{\Rbb}{R_{\mbox{\tiny BB}}}
\newcommand{\Rsim}{R_{\mbox{\tiny sim}}}
\newcommand{\oracle}{_{\mbox{\tiny oracle}}}
\newcommand{\MKL}{M_{K,L}}
\newcommand{\pmin}{p_{\min}}
\newcommand{\vmin}{\v_{\min}}
\newcommand{\ThM}{\eta}
\newcommand{\sigeff}{\sigma_{\operatorname{eff}}}

\newcommand{\vp}{\sigma}
\newcommand{\subvar}{\tilde{\sigma}}
\newcommand{\subG}{\operatorname{subG}}

\newcommand\blfootnote[1]{%
	\begingroup
	\renewcommand\thefootnote{}\footnote{#1}%
	\addtocounter{footnote}{-1}%
	\endgroup
}

\begin{document}

\title{Distributed Sparse Normal Means Estimation with Sublinear Communication}

\author{{
		Chen Amiraz\thanks{Corresponding author: chen.amiraz@weizmann.ac.il},  Robert Krauthgamer and Boaz Nadler},\\[2pt]
	Department of Computer Science, Weizmann Institute of Science, Rehovot, Israel\\
}

\maketitle

\begin{abstract}
        {We consider the problem of sparse normal means estimation in a distributed setting with communication constraints. We assume there are $M$ machines, each holding $d$-dimensional observations of a $K$-sparse vector $\muv$ corrupted by additive Gaussian noise.
        {The $M$ machines are connected in a star topology to a fusion center, whose goal is to estimate the vector $\muv$ with a low communication budget.}
        Previous works have shown that to achieve the centralized minimax rate for the $\ell_2$ risk, the total communication must be high -- at least linear in the dimension \(d\). This phenomenon occurs, however, at very weak signals. {We show that at signal-to-noise ratios (SNRs) that are sufficiently high -- but not enough for recovery by any individual machine -- the support of $\muv$ can be correctly recovered with significantly less communication.         }
        {Specifically, we present two algorithms for distributed estimation of a sparse mean vector corrupted by either Gaussian or sub-Gaussian noise. We then prove that above certain SNR thresholds, with high probability, these algorithms recover the correct support with total communication that is {\it sublinear} in the dimension $d$.} Furthermore, the communication decreases {\it exponentially} as a function of signal strength. If in addition $KM\ll \tfrac{d}{\log d}$, then with an additional round of sublinear communication, our algorithms achieve the centralized rate for the $\ell_2$ risk.
        Finally, we present simulations that illustrate the performance of our algorithms in different parameter regimes.
}

\blfootnote{{This article has been accepted for publication in Information and Inference: A Journal of the IMA, Published by Oxford University Press.}}

\end{abstract}

\section{Introduction}
In the past couple of decades, the steady increase in data collection capabilities has lead to rapid growth in the size of datasets. 
In many applications, the collected datasets cannot be stored or analyzed on a single machine, which has sparked the development of distributed approaches for machine learning, statistical analysis, and data mining. A few examples of this vast body of work are \citep{mcdonald2009efficient,bekkerman2011scaling,duchi2012dual,guha2012large}.

One of the most popular distributed settings is known as one-shot, embarrassingly parallel 
or split-and-merge.
In this setting, there are $M$ machines, each holding an independent set of samples from some unknown distribution, connected in a star topology to a central node, also called a fusion center or simply the center.
The task of the fusion center is to estimate $\theta$, a parameter of the distribution, using {little} communication with the $M$ machines. 
	In one-shot schemes, there is only a single round of communication. 
	The fusion center may send a setup message to the machines (or a subset of them). 
Then, each contacted machine performs a local computation and sends its result back to the center.  
Finally, the fusion center forms a global estimator $\hat\theta$ based on these messages.
A clear advantage of such one-shot schemes is their simplicity and ease
of implementation.

Statistical inference in a distributed setting, in particular under communication constraints, raises several
fundamental theoretical and practical questions. One question is what is the loss {in} statistical accuracy incurred by distributed schemes, compared to a centralized setting, whereby a single machine has access to all of the samples. 
Various works proposed multi-round communication-efficient schemes and analyzed their accuracy, see for example \citep{shamir2014communication,zhang2015disco,wang2017efficient,jordan2019communication}.
In the context of one-shot schemes, several works analyzed the case where the fusion center
simply averages the estimators computed by the individual machines or for robustness, takes their median
\citep{zhang2013communication,rosenblatt2016optimality,minsker2019distributed}.
In a high dimensional setting where the parameter of interest is a-priori known to be sparse, \citet{lee2017communication} and \citet{battey2018distributed} considered a variant where the averaged estimator
is further thresholded at the fusion center. 
A key {finding} in many of these papers is that in various scenarios and under suitable regularity assumptions, the $\ell_2$ risk of the distributed estimate {attains the same convergence rate} as the centralized one, provided that the data is not split across too many machines.

Another important theoretical aspect in distributed learning is fundamental lower bounds on the achievable accuracy under communication as well as memory constraints, regardless of any specific inference scheme, see e.g. \citep{zhang2013information,garg2014communication,steinhardt2016memory,cai2020distributed,zhu2018distributed,szabo2020adaptive,acharya2020general},
and similarly for the closely related problem of distributed detection \citep{acharya2020distributed,szabo2020optimal}.
Lower bounds on the estimation accuracy were also studied for problems involving a sparse quantity,
including sparse linear regression, correlation detection and more \citep{steinhardt2015minimax,braverman2016communication,dagan2018detecting,han2018geometric}.
A central {finding} in these works is that 
to achieve the centralized minimax rate for the $\ell_2$ risk, the communication must scale at least linearly in the ambient dimension. 

{However,} when the task is to estimate a sparse quantity, {then intuitively} the communication should increase 
linearly with its sparsity level, and only logarithmically with the ambient dimension. 
Indeed, in the context of supervised learning,
\citet{acharya2019distributed} showed that in various linear models with a sparse vector, 
optimal {prediction} error rates are achievable with total communication logarithmic in the dimension. 
However, {they consider connectivity topology of a chain where each machine $i$ sends a message only to machine $i+1$, and thus} their algorithm is sequential and not compatible with one-shot inference schemes.
An interesting question is the following: can problems that involve a sparsity prior admit one-shot algorithms with communication that is \textit{sublinear} in the ambient dimension? 

We consider sparse normal means estimation, which is one of the simplest and most well-studied inference problems with sparsity priors, but in a distributed setting {of \(M\) machines connected in a star topology to a fusion center.
For simplicity we assume that each machine has the same number $n$ of i.i.d. samples of the form $\x_i=\muv + \errv_i$, where the mean vector $\muv\in \R^d$ is exactly $K$-sparse and the noise is Gaussian, $\errv_i\sim\Nor (\bm{0},\sigma^2 \bm{I}_d)$. 
The assumption of Gaussian noise implies that the empirical mean is a sufficient statistic.
Thus, in our analysis we may equivalently assume that each machine has only one independent observation of $\muv$ with an effective noise level of $\sigeff = \tfrac{\sigma}{\sqrt{n}}$.
In addition, we assume for simplicity that the noise level $\sigma $ is known. Hence, without loss of generality, 
 we assume that the single observation at each machine has noise level $\sigeff=1$.}

{We consider a one-shot communication scheme where the fusion center sends a setup message to each of the machines (or a subset of them), and then each contacted machine sends back its message to the center. 
	We emphasize that in our setting the machines communicate only with the center and not with each other. 
Note that if the machines have prior knowledge of all problem parameters, then setup messages are not required. 
However, in any case the communication of this setup stage is often negligible.
The goal of the center is to recover the support of $\muv$ under the constraint that the total communication between the fusion center and the machines (including the setup stage) is bounded by a budget of $B\ll d$ bits. 
As we discuss in Section \ref{sec:ell_2}, if $KM\ll \tfrac{d}{\log d}$ then achieving this goal implies that the vector $\muv$ itself can be estimated with small $\ell_2$ risk using communication sublinear in $d$.}

For this sparse normal means problem, \citet{braverman2016communication} and \citet{han2018geometric} derived communication lower bounds for the $\ell_2$ risk of any estimator,
and proved that to achieve the minimax rate, the total
communication must be at least $\Omega(d)$. 
{
	\citet{shamir2014fundamental} derived lower bounds for several other distributed
	problems involving $M$ machines, each allowed to send a message of length at most $b$ bits. His work implies that there exist $d$-dimensional distributions whose mean is a $1$-sparse vector
   of sufficiently low magnitude,	such that with 
	$n=O(d\log d)$ samples per machine, any scheme with communication sublinear in $d$
	has only an $o(1)$ probability of exact support recovery.}
These works paint a pessimistic view, that to achieve the performance of the centralized solution, distributed inference must incur high communication costs.

{In contrast, our main contribution is to show that at SNRs that are sufficiently high, but not high enough for recovery by any individual machine, the support of $\muv$ can be exactly recovered with total communication {\it sublinear} in the dimension $d$. 
Specifically, we present and analyze the performance of two distributed schemes.
Our analysis is non-asymptotic, but the setting we have in mind is of a sparse vector in high dimension, namely $d\gg 1$ and $K\ll d$. 
Assuming that $\mumin$, a lower bound on the non-zero entries of $|\muv|$, is known to the center and exceeds $\Omega\left( \tfrac{ \log\log d }{\sqrt{\log d}} \right) $, and that the number of machines $M$ is sufficiently high, we prove the following results.
First, with high probability, our two schemes recover the support of $\muv$ with total communication \textit{sublinear} in $d$. 
Second, since the center need not contact all $M$ machines, the communication costs of our proposed schemes decrease \textit{exponentially} as $\mumin$ increases towards $\sqrt{2\log d}$, at which point the support of $\muv$ may be found by a single machine using $O(K\log d)$ communication bits. 
Third, we present the following 
counter-intuitive behavior of our algorithm: \textit {more machines enable less communication}.
Specifically, as discussed after Theorem \ref{thm:th_M}, for some range of the problem parameters, as the number of machines is increased, exact support recovery is possible with \textit{less} total communication.  
We further extend some of these results to the case of sub-Gaussian additive noise.
Finally, we prove that if $KM\ll \tfrac{d}{\log d}$, then an additional single round of communication, also sublinear in $d$, results in an estimator for $\muv$ that achieves the centralized rate for $\ell_2$ risk. }

{This idealized setting allows for a relatively simple analysis that showcases a tradeoff between the number of machines, SNR, and communication.
	Four remarks are in place. 
	First, it remains an open problem whether the SNR-communication tradeoff of our algorithms is optimal.
	Indeed, the derivation of tight SNR-dependent communication lower bounds for the sparse normal means problem is an interesting topic for future research.
	Second, we focus on the simple case where all machines have the same number of samples $n$ and all samples have the same noise level $\sigma$. An interesting direction for future research is to consider a more general setting where each machine $i$ has a different number of samples $n_i$, or a different noise level $\sigma_i$. Another interesting setting is where each machine observes different sparse vectors $\muv_i$ with the same support $S$ (or very similar supports $S_i$). Note that there is no single SNR parameter in these cases since different machines have different effective SNRs.
	Third, the estimator in our idealized setting is linear and thus unbiased. This avoids the added complication of analyzing the bias-variance tradeoff in a distributed setting.  
	Lastly, building on the insights gained in this simple setting, we believe a similar behavior should hold for other popular statistical learning problems involving estimation of a sparse quantity in a high dimensional setting. 
}

\paragraph{Paper organization.}
In Section \ref{sec:problem_setup} {we characterize the SNR regime relevant to the distributed sparse normal means problem.}
{Section \ref{sec:algs} presents several algorithms for exact support recovery, for either Gaussian or sub-Gaussian additive noise.
	We then prove that under suitable assumptions on the SNR and on the number of machines, our algorithms achieve exact support recovery with high probability using sublinear total communication in the ambient dimension $d$.}
Section \ref{sec:ell_2} discusses the relation between exactly recovering the support of a vector and estimating it with small $\ell_2$ risk, and shows a reduction from the latter to the former with one additional round of sublinear communication.
Section \ref{sec:prev} elaborates on how our results relate to the lower bounds of \citet{braverman2016communication}, \citet{han2018geometric} and \citet{shamir2014fundamental}. Section \ref{sec:sim} presents simulations that illustrate our results. All proofs can be found in the appendix.

\paragraph{Notation.}
{We use the standard $O(\cdot),\Omega(\cdot),\Theta(\cdot)$ notation to hide constants independent of the problem parameters and the notation $\tilde{O}(\cdot)$ to hide terms that are at most polylogarithmic in $d$. 
	For functions $f,g$ the notations $f=o(g)$ and $f\ll g$ imply that $f/g\to 0$ as $d\to\infty$.
	The term exact recovery of the support $\indset$ with high probability means that an estimator $\hat{\indset}$ correctly estimates the support, i.e., $\Pr\left[ \hat{\indset}=\indset\right] \to 1$ as $d\to\infty$ and the number of machines $M=M(d)$ tends to infinity at a suitable rate, as detailed in each theorem.
	We use the notation $\lceil x \rceil$ for the smallest integer larger than or equal to $x$.}

\section{SNR regime}\label{sec:problem_setup}
{As mentioned in the introduction, we assume each of $M$ machines has $n$ samples corrupted by additive Gaussian noise of known noise level. Hence, without loss of generality we assume that each machine $i$ stores a single observation $\x_i=\muv+\errv_i$, where $\errv_i\sim\Nor\left( \bm{0},\bm{I}_d\right) $ and $\muv$ is exactly $K$-sparse.}
For simplicity we assume that the sparsity level $K$ is known to the fusion center and that $\mu_j\geq 0$ for all $j\in[d]$. However, with slight variations our methods can work when $K$ is unknown or for vectors $\muv$ that have both positive and negative entries.
We further assume a lower bound $\mumin$ on its smallest non-zero coordinate, namely $\mu_j \geq \mumin$ for all $j\in\indset=\{i\,|\mu_i > 0\}$. 
It will be convenient to use the natural scaling 
\begin{equation}
\mumin = \sqrt{2 r \log \left( d-K\right) }.
        \label{eq:mu_min_r}
\end{equation}
{We focus on the following question: Given a lower bound on the signal-to-noise ratio (SNR) $r$, how much communication is sufficient for \textit{exact recovery of the support $\indset$} of a $K$-sparse vector $\muv$ with high probability?}

Let us first discuss what is the interesting regime for the SNR parameter $r$.
Recall that for $d \gg K$, the maximum of $d-K$ i.i.d. standard Gaussian random variables is tightly concentrated around $\sqrt{2 \log \left( d-K\right)}$. At a high SNR $r>1$, each individual machine can thus exactly recover the support set $\indset$ with high probability. Hence, it suffices that only one machine sends $O(K\log d)$ bits to the fusion center. At the other extreme, let $r<\frac{c}{M}$ for a fixed $0<c<1$.
Here, even in a centralized setting, exact support recovery with high probability is not possible. To see this, note that the empirical mean of all samples is a sufficient statistic, and {its} effective SNR is $c<1$. Therefore, with probability tending to $1$ as $d\to\infty$, {its} smallest support entry is smaller than its largest non-support entry. If the index of $\mumin$ is chosen uniformly at random, then any algorithm would fail to recover the support.
{Hence, the relevant SNR values are}
\begin{equation} \label{eq:SNR_regime}
\frac1M <r<1.
\end{equation}
In this range, a single machine cannot individually recover the support with high probability. 
Yet, as we show next, for a large {subrange of the SNR values given in Eq.} \eqref{eq:SNR_regime}, exact support recovery by the fusion center is possible with total communication $o(d)$ bits.  Furthermore, as $r$ increases towards $1$, the total communication decays exponentially fast to $O\left( K \log^{1+c} d\right) $ for an appropriate constant $c>0$.

\section{Distributed algorithms for the sparse normal means problem}\label{sec:algs}

We present two one-shot algorithms for the distributed sparse normal means problem and derive non-asymptotic bounds on their performance, {namely, their probability of exact recovery and their total communication.} For both algorithms, the lower bound $r$ on the SNR is assumed to be known to the center and is used to decide how many machines to communicate with and what messages to send them. {We use the notation $\Mc$ for the number of contacted machines, which is different in each theorem.}
For our analysis below, we assume the total number of machines is sufficiently large, in particular $M\geq \Mc$, {which is a stronger condition than the centralized lower bound $M>1/r$.}

In our first algorithm, denoted Top-$L$, the center sends a parameter $L$ to $\Mc$ machines. Each contacted machine $i$ sends back a message $\y_i$ with the indices of the $L$ highest coordinates of its sample $\x_i$.
Our second algorithm is threshold-based; the center sends a threshold $\tm$ to $\Mc$ machines, and each contacted machine $i$ sends back all indices $j$ with $x_{i,j} > \tm$. 
In {both} two algorithms, the center then estimates the support of $\muv$ by a voting procedure. 
{We prove in Theorems \ref{thm:top} and \ref{thm:th} that under suitable assumptions, and in particular for a sufficiently 
high SNR, both algorithms achieve exact support recovery with high probability using sublinear communication.
In particular, we show in Theorem \ref{thm:th_logd} that if $r>\Omega\left( \log^{-1} \left( d-K\right) \right)$, then with high probability 
the thresholding algorithm with $\Mc=O(\log d)$ machines and $\tm=\sqrt{2r\log (d-K)}$ recovers the support of the $K$-sparse vector $\muv$ using  
$\tilde{O}\left(\left( d-K\right)^{1-r}+K\right)$
communication bits in expectation. 
The total communication cost is sublinear in $d$ provided that $K \ll d$ and $r>\Omega\left(\frac{\log \log d}{\log d} \right) $. 
Moreover, increasing the threshold allows for a tradeoff between $M_c$ and the expected message length per machine. As we show in Theorems \ref{thm:th_M} and \ref{thm:th_large_M}, perhaps counter-intuitively, given more than $O\left( \log d\right) $ machines, the fusion center can recover the support using \textit{less} total communication, by setting a higher threshold.
Specifically, if $r>\Omega\left( \log^{-2} \left( d-K\right) \right)$, then with high probability the thresholding algorithm with $\Mc=\tilde{O}\left(\left(d-K\right)^{\left(1-\sqrt{r}\right)^{2}}  \right)$ machines and
$\tm=\sqrt{2 \log \left( d-K\right)}$ recovers the support of the $K$-sparse vector $\muv$ using $\tilde{O}\left( K \left(d-K\right)^{\left(1-\sqrt{r}\right)^{2}} \right) $ communication bits. Note that the resulting total communication cost is sublinear in $d$, provided that $K$ is at most polylogarithmic in $d$ and 
$r>\Omega\left( \frac{\log^2\log d}{\log^2 d}\right)   $. 
We also prove a similar result for the Top-$L$ algorithm with $L=K$ in Theorem \ref{thm:top}.
Finally, in Section \ref{sec:sub-G} we extend some of these results to the case of additive sub-Gaussian noise.
}

\begin{figure}[t]
	\centering
	\includegraphics[width=\linewidth]{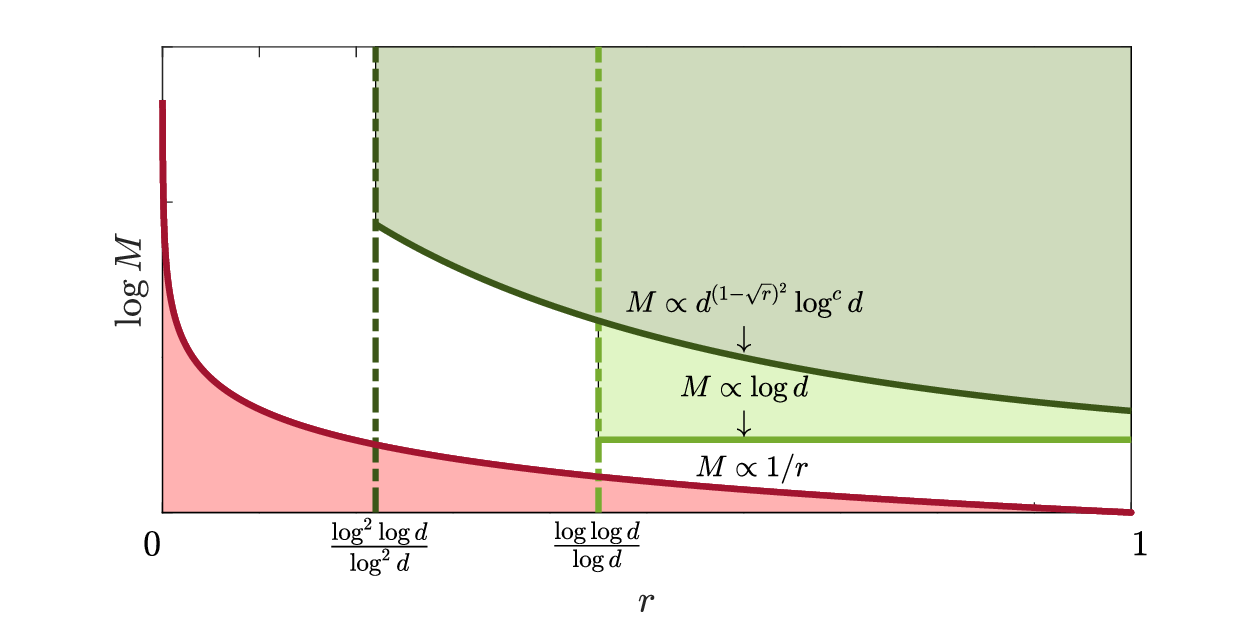}
	\caption{Illustration of communication regimes as a function of $\left(r,M \right)$, where $M$ is on a logarithmic scale. 
		{In the pink area, exact support recovery is information-theoretically impossible even in the centralized setting.
			In the green areas, our distributed algorithms achieve exact support recovery with communication that is sublinear in $d$.}
}
	\label{fig:M_vs_r}
\end{figure}

To put our results in context, we illustrate in Figure \ref{fig:M_vs_r} the different communication regimes as a function of the SNR $r$ and the number of machines $M$ for $K=1$.
As discussed above, if $r< \frac{1}{M}$, then even with infinite communication, exact support recovery with high probability is information-theoretically impossible. The corresponding $\left(r,M \right) $ values are in the pink area below the red curve {which delineate the relation $r\cdot M =1$. }
{
By our Theorems \ref{thm:top} and \ref{thm:th}, exact recovery with sublinear communication is possible in the light green and dark green areas.}
In the white area, {distributed} exact support recovery is possible using communication that is at least linear in $d$. An example of a recovery scheme in this range is to send the entire sample (up to a quantization error). 
It remains an open question whether exact support recovery with sublinear communication is possible for  $\left(r,M \right) $ values in the white area.

\subsection{Top-$L$ Algorithm}
In the Top-$L$ algorithm, the center uses its knowledge of the parameters $d,M,r,K$ to determine the number of machines $\Mc$ to contact, and sends them a parameter $L\in \mathbb{N}$. The $i$-th contacted machine then sends a message $\y_i$ consisting of the $L$ indices with the largest coordinates of its vector $\x_i$. Given the messages $\y_1,\dots,\y_{\Mc}$, the fusion center counts how many votes each index received and estimates the support to be the $K$ indices with the highest number of votes. Voting ties can be broken arbitrarily.
This scheme is outlined in Algorithm \ref{alg:topL}. Its total communication cost is $B=O(L\Mc\log d)$ bits.

\begin{remark}\label{rem:unknownK}
	The above description assumes that the fusion center knows the sparsity level $K$. However the following simple variant can handle a case where only an upper bound $K_{\max}\geq K$ is known. In this case, the number of contacted machines $\Mc$ is determined using $K_{\max}$ instead of $K$, and each contacted machine sends its top $L\geq K_{\max}$ indices to the fusion center.
	The center then estimates the support as the set of indices that received more votes than a suitable threshold {$\tc\left( d\right) $ (see Eq. \eqref{eq:tc}). }
\end{remark}

\begin{algorithm}[h]
        \caption{Top-$L$}
        \label{alg:topL}
        \vspace{6pt}
        \noindent\textbf{At the fusion center:}\\
        \textbf{Input} dimension $d$, number of machines $M$, SNR $r$, sparsity level $K$, parameter $L$\\
        \textbf{Output} setup message $\s$ 
        \begin{algorithmic}[1]     
        	\State  {if $L=K=1$, then $\Mc $ is given by Eq. \eqref{eq:M_0}, otherwise it is given by Eq. \eqref{eq:M_K_L} }
        	\State  {send message $\s$ that contains the value $L$ to each of the first $\Mc$ machines} \vspace{6pt}
        \end{algorithmic}
        \textbf{At each machine $i=1,\dots,\Meff$:}\\
        \textbf{Input} 	setup message $\s$, sample $\x_i$\\
        \textbf{Output} message $\y_i$ to center
        \begin{algorithmic}[1]     
                \State  {compute permutation $\alpha_i$ that sorts the vector $\x_{i}$ in descending order, $x_{i,\alpha_i\left( 1\right) }\geq \dots \geq x_{i,\alpha_i\left( d\right) }$}
                \State  send to the center the $L$ indices with the largest coordinates, $\y_i=\left\lbrace  \alpha_i\left(1\right) ,\dots ,\alpha_i\left(L\right)\right\rbrace   $ \vspace{6pt}
        \end{algorithmic}
        \textbf{At the fusion center:}\\
        \textbf{Input} messages $\y_1,\dots,\y_{\Meff}$, sparsity level $K$\\
        \textbf{Output} estimated support $\hat{\indset}$
        \begin{algorithmic}[1]     
                \State  for each coordinate $j\in \left[ d\right] $, let the set of votes it received be $\V_j=\left\lbrace i\in \left[ \Mc\right] : j\in\y_i\right\rbrace $ {and let their number be $\v_j=\left| \V_{j}\right|$}
                \State  {compute permutation $\pi$ that sorts the indices by descending number of votes, $\v_{\pi(1)}  \geq \dots \geq \v_{\pi(d)} $ }
                \State return $\hat{\indset}=\left\lbrace \pi(1),\dots, \pi(K) \right\rbrace $
        \end{algorithmic}
\end{algorithm}

We prove that for sufficiently high SNR, the Top-$L$ algorithm recovers the exact support of $\muv$ with high probability.
{To ease the presentation and highlight the main ideas of the proof, we first analyze the case $L=K=1$ and then extend the analysis to general $L\geq K \geq 1$.} 
The proofs of the theorems stated below appear in Appendix \ref{sec:proof_thm_top}. 

{Motivated by the required number of machines for proving Theorem \ref{thm:top1},} we define the quantity
\begin{equation}\label{eq:M_0}
        M_{0}\left(d,r\right)=\left\lceil \max\left\{ 1,\frac{\sqrt{2\pi}e\left(2\left(1-\sqrt{r}\right)^{2}\log d+1\right)}{\left(1-\sqrt{r}\right)\sqrt{2\log d}}\cdot d^{\left(1-\sqrt{r}\right)^{2}}\right\}\cdot 8\log d\right\rceil.
\end{equation}
Notice that for any fixed SNR $r<1$, $M_{0}\left(d,r\right)$ is \textit{sublinear} in $d$, and up to polylogarithmic terms it is proportional to $d^{\left(1-\sqrt{r}\right)^{2}}<d$. 
The following theorem provides a support recovery guarantee in the setting $K=L=1$.
\begin{subtheorem}{thm}\label{thm:top}
        \begin{thm}\label{thm:top1}
                Assume $r<1$ and that $ M_0(d,r) \leq \min\left\lbrace  M,d \right\rbrace $. Then, if the center contacts $\Meff=M_0$ machines, 
                the Top-$1$ algorithm recovers the support of a $1$-sparse vector $\muv$ with probability at least $1-d^{-1}-e^{3}d^{-3}$. Its total communication is $O(M_0\log d)$ bits. 
        \end{thm}

Several insights follow from Theorem \ref{thm:top1}. 
First, recall that for any $r<1$ no machine can successfully recover the support of $\muv$ on its own. 
Yet, for $d\gg 1$ and for any fixed $r<1$, as implied by the theorem, the fusion center can recover the support of $\muv$ by communicating with only $M_0\left( d,r\right) $ machines, receiving from each machine {its own mostly inaccurate} estimate of the support.
Second, as the SNR lower bound $r$ increases towards $1$, the algorithm needs {to contact} fewer machines and thus less communication to succeed with high probability. Moreover, by Eq. \eqref{eq:M_0}, $M_0\left( d,r\right) $ decreases \textit{exponentially} fast with $r$. 
Lastly, for a fixed $r$ the required number of machines $M_0\left( d,r\right) $ and thus the total communication cost both increase sublinearly with $d$.

Next, we consider the more general case where the unknown vector $\muv$ is exactly sparse with sparsity level at most $K$, and its support is estimated by the Top-$L$ algorithm with parameter $L\geq K$.
To this end, we define the auxiliary quantities
\begin{equation}\label{eq:a}
a=a\left(K,L,d \right) =\sqrt{2\log\frac{d-K}{L-K+1}},
\end{equation}
\begin{equation}\label{eq:b}
b=b\left(K,L,d,r \right)=a-\sqrt{2r\log\left(d-K\right)},
\end{equation}
and the quantity
\begin{equation}\label{eq:M_K_L}
	\MKL\left(d,r\right)=\left\lceil \max\left\{ 1,\frac{4\sqrt{2\pi}\left(b^{2}+1\right)}{b}\cdot\left(d-K\right)^{\left(\sqrt{1-\frac{\log\left(L-K+1\right)}{\log\left(d-K\right)}}-\sqrt{r}\right)^{2}}\right\} \cdot8\log d\right\rceil   .
\end{equation}

The following theorem provides a support recovery guarantee in this setting.
\begin{thm}\label{thm:topL}
        Assume $r<1$ and that $\MKL(d,r)\leq \min\left\lbrace M,\frac{d-K}{L} \right\rbrace    $. Then, if the center contacts $\Meff=\MKL$ machines, 
        the Top-$L$ algorithm with $K \leq L < (d-K)/2 $ recovers the support of a $K$-sparse vector $\muv$ with probability at least $1-Kd^{-1}-e^{3}d^{-3}$ using $O\left( L\MKL  \log d\right) $ communication bits. 
\end{thm}

\end{subtheorem}
{While the expressions in Theorem \ref{thm:topL} are more involved than those of Theorem \ref{thm:top1}, similar insights to those mentioned above continue to hold.
In particular, the Top-$L$ algorithm with $L=K$ incurs a total communication cost of $O\left( K\cdot \left( d-K\right) ^{\left(1-\sqrt{r}\right)^{2}}\log^{2.5}d\right) $, which is sublinear in $d$ provided that $K$ is at most polylogarithmic in $d$ and $r>\Omega\left(\frac{\log^2\log d}{\log^2 d} \right) $.
}

\begin{remark}
	One can consider a variant of the algorithm that sends $L<K$ randomly selected indices out of the top-$K$. This randomized variant would allow a tradeoff between the number of contacted machines and the message length per machine.
\end{remark}

\subsection{Thresholding Algorithm}
{In our second algorithm,} the fusion center chooses a threshold $\tm=\tm\left( d,r,M,K\right)$ and sends (a truncated binary representation of) it to a subset of the machines $\Meff=\Meff\left( d,r,M,K\right) \leq M$. Each contacted machine $i$ sends back all indices $j$ such that $x_{i,j}>\tm$. Similarly to the Top-$L$ algorithm, given the messages $\y_1,\dots,\y_{\Mc}$ and the sparsity level $K$, the fusion center estimates the support as the $K$ indices with the highest number of votes. Voting ties can be broken arbitrarily. The scheme is outlined in Algorithm \ref{alg:th}.
If instead of the sparsity level $K$ only an upper bound on it $K_{\max}\geq K$ is known, and $K_{\max}\ll d$, then the fusion center can set $\tm$ and $\Meff$ by approximating $d-K \approx d$. In addition, the center estimates the support as outlined in Remark \ref{rem:unknownK}.

\begin{algorithm}[t!]
        \caption{Thresholding}
        \label{alg:th}
        \vspace{6pt}
        \noindent\textbf{At the fusion center:}\\
        \textbf{Input} 	dimension $d$, number of machines $M$, SNR $r$, sparsity level $K$\\
        \textbf{Output} setup message $\s$
        \begin{algorithmic}[1]     
        	\State  {depending on $M$, calculate $\Meff$ and $\tm$ by their expressions in Theorem \ref{thm:th_logd} or \ref{thm:th_M} or \ref{thm:th_large_M}}
        	\State  send {message} $\s=\Trunc(\tm,\left\lfloor\log_2 \tm \right\rfloor,\left\lceil \log_2 d \right\rceil)$ to each of the first $\Meff$ machines \vspace{6pt}
        \end{algorithmic}
        \textbf{At each machine $i=1,\dots,\Meff$:}\\
        \textbf{Input} 	setup message $\s$, sample $\x_i$\\
        \textbf{Output} message $\y_i$
        \begin{algorithmic}[1]     
        		\State	construct threshold $\htm=\Approx(\s,\left\lfloor\log_2 \tm \right\rfloor,\left\lceil \log_2 d \right\rceil)$
                \State 	let $\y_i=\left\lbrace j\in \left[ d\right] : x_{i,j}>\htm \right\rbrace $
                \State  send $\y_i$ to center \vspace{6pt}
        \end{algorithmic}
        \textbf{At the fusion center:}\\
        \textbf{Input} 	messages $\y_1,\dots,\y_{\Meff}$, sparsity level $K$\\
        \textbf{Output} estimated support $\hat{\indset}$
        \begin{algorithmic}[1]     
                \State  for each coordinate $j\in \left[ d\right] $, let the set of votes it received be $\V_j=\left\lbrace i\in \left[ \Mc\right] : j\in\y_i\right\rbrace $ {and let their number be $\v_j=\left| \V_{j}\right|$}
                \State  {compute permutation $\pi$ that sorts the indices by descending number of votes, $\v_{\pi(1)}  \geq \dots \geq \v_{\pi(d)} $}
                \State return $\hat{\indset}=\left\lbrace \pi(1),\dots, \pi(K) \right\rbrace $
        \end{algorithmic}
\end{algorithm}

The thresholding algorithm has several desirable properties. 
First, it is simple to implement in a distributed setting.
Second, in the centralized setting, thresholding algorithms were shown to be optimal in various aspects (see Section \ref{sec:ell_2} for further details).
Third, adjusting the threshold allows for a tradeoff between the number of contacted machines and the expected message length per machine.
{Notice that if the SNR is sufficiently high, but still $r<1$, i.e., not high enough for recovery by any individual machine, there may not even be a need to contact all machines to recover the support.} 
By the same logic, when the SNR is lower, one can lower the threshold. Of course, this would incur a higher communication cost.
Hence, since the fusion center knows both $r$ and $M$, it can set an optimal threshold $\tm$ and send it only to $\Meff\leq M$ machines, which ensures exact support recovery with high probability at minimal communication cost (among all possible thresholds).

{To complete the description of the algorithm, we now describe our approximation of a real number by a finite amount of bits.
	Recall that the scientific binary representation of a number $x\in\R$ consists of a bit representing its sign and bits $\left\lbrace b_j\right\rbrace_{j\in\mathbb{Z}} $, such that $|x|=\sum_{j=-\infty}^{\left\lfloor\log_2|x|\right\rfloor}b_{j} 2^j$.
	One can approximate $x$ by truncating its binary representation at a predetermined precision level.  
	Specifically, given two parameters $U,P\in\mathbb{N}$, let the procedure $\s=\Trunc(x,U,P)$ output a truncated binary representation of $x$ of length $U+P+2$ such that $\s=\left( \mathds{1}\left\lbrace x\geq 0\right\rbrace,b_{-P},\dots,b_{U} \right) $. Given $\s$, let the procedure $\hat{x}=\Approx(\s,U,P)$ construct an approximation for $x$, given by $\hat{x}= \sign(x)\cdot \sum_{j=-P}^{U}b_{j} 2^j$. 
	If $U\geq \left\lfloor\log_2|x|\right\rfloor$, then $\hat{x}$ and $x$ consist of the same bits up to the $P$-th bit after the binary dot, and thus the resulting approximation error is bounded by $|\hat{x}-x|<2^{-P}$.
	This scheme is a variant of \citet[Algorithm 1]{szabo2020adaptive}. }

{In our analysis we assume that $\mumax=\max_{j\in\indset}\mu_j$ is at most polynomial in $d$.
	Thus taking $U,P=O(\log d)$ ensures that with high probability all quantities of interest are approximated up to $d^{-O(1)}$ error.
	In addition, since $P,U$ only depend on $d$ and on the bound $\mumax$, they can be set in advance without communication.
}

{We analyze the performance of the thresholding algorithm in three regimes, in terms of the number of contacted machines $M_c$: small, intermediate, and large (clearly under the constraint that $M\geq M_c$).
For each regime, we derive a different threshold $\tm=\tm\left( d,r,M\right) $, where the SNR parameter $r$ and sparsity level $K$ are assumed to be known.
In the small $M_c$ regime, considered in Theorem \ref{thm:th_logd}, the number of contacted machines is logarithmic in $d$. The corresponding threshold $\tm$ given by \eqref{eq:tm_logd} is relatively small.
In the intermediate regime, considered in Theorem \ref{thm:th_M}, all $M$ machines are contacted and the threshold $\tm$, given by Eq. \eqref{eq:tm}, increases as a function of $M$.
Finally, when the number of available machine is sufficiently large, 
as described in Theorem \ref{thm:th_large_M}, the center
contacts only a subset $M_c$ of all machines, where the value of $M_c$ is chosen to minimize the total communication, while still achieving exact support recovery with high probability.  
The proofs appear in Appendix \ref{sec:proof_thm_th}.}
\begin{subtheorem}{thm}\label{thm:th}
	\begin{thm} \label{thm:th_logd}
		Assume that $d\geq 16$ and $M\geq 16\log d $. Further assume $\frac{\log 5}{\log \left( d-K\right) }<r < 1$.
		Then, with probability at least $1-\left( K+1\right)/ d$, the thresholding algorithm with $\Meff =  \left\lceil 16\log d \right\rceil $ and
		\begin{equation}\label{eq:tm_logd}
		\tm = \sqrt{2 r \log \left( d-K\right)}
		\end{equation} 
		recovers the support of the $K$-sparse vector $\muv$ using  
		\begin{equation}\label{eq:th_comm_logd}
		O\left(\left( d-K\right)^{1-r}r^{-0.5}\log^{1.5}d+K\log^{2}d\right)
		\end{equation}
		communication bits in expectation. 
	\end{thm}
The communication cost \eqref{eq:th_comm_logd} is sublinear in $d$ for all $r>\frac{2\log \log d}{\log \left( d-K\right) }$ and $K\ll d/\log^{2}d$.
{Note that in the above theorem, the number of contacted machines is fixed at $16\log d$ and correspondingly, the threshold does not depend on the total number of machines $M$. 
The next theorem shows that contacting all machines with a higher threshold that depends on the total number of machines, can lead to exact support recovery with even less communication than \eqref{eq:th_comm_logd}.
}

\begin{thm} \label{thm:th_M}
        Let $d\geq 15$ and assume that $32\sqrt{e\pi}\log^{1.5}d \leq M \leq d$. Further assume $r<1$ and that
        \begin{equation}\label{eq:r_th_M}
        r>\frac{\left(\sqrt{2\log\frac{5M}{\sqrt{2\pi}4\log d}}-\sqrt{2\log\frac{M}{32\sqrt{\pi}\log^{1.5}d}}+\frac 1d\right)^{2}}{2\log\left(d-K\right)}.
        \end{equation}
        Then, with probability at least $1-\left( K+1\right)/ d$, the thresholding algorithm with $\Meff = M$ and
        \begin{equation}\label{eq:tm}
        \tm=\sqrt{2 r \log \left( d-K\right)}+\sqrt{2\log{\frac{M}{32\sqrt{\pi}\log^{1.5}d}} }
        \end{equation} 
        recovers the support of the $K$-sparse vector $\muv$ using  
        \begin{equation}\label{eq:th_comm}
        O\left(KM\log d+\left(d-K\right)^{1-r}e^{-2\sqrt{r\log\left(d-K\right)\log\frac{M}{32\sqrt{\pi}\log^{1.5}d}}}\log^{2.5}d\right)
        \end{equation}
        communication bits in expectation. 
\end{thm}
{
It is interesting to study the behavior of the total communication cost in Eq.
\eqref{eq:th_comm}. The first term increases with $M$, whereas the second term decreases with $M$. It is easy to show that the total communication cost is
minimized at $M_{\mbox{\tiny opt}} = \tilde{O}\left(\left( d-K\right) ^{\left( 1-\sqrt{r}\right) ^2} \right)$. 
This leads to a perhaps 
 counter-intuitive result, that 
 in the range $O\left( \log^{1.5}d\right)  < M < M_{\mbox{\tiny opt}}$, 
  as the number of machines increases 
 exact recovery is possible with \textit{less} total communication. 
 Once the number of available machines is larger than $M_{\mbox{\tiny opt}}$, there is no benefit in contacting all machines. In terms of total communication, it is best to simply contact $M_c=M_{\mbox{\tiny opt}}$ of them, as stated in the following theorem. 
}
\begin{thm}\label{thm:th_large_M}
	Assume that $d-K\geq 20$ and $\left( \frac{\log 10}{\log \left( d-K\right) }\right) ^2<r<1$. Let 
	\begin{equation}\label{eq:Mlarge}
	\Meff=\left\lceil  \frac{8\sqrt{2\pi}\left(\left(1-\sqrt{r}\right)^{2}2\log\left(d-K\right)+1\right)}{\left(1-\sqrt{r}\right)\sqrt{2\log\left(d-K\right)}}\left(d-K\right)^{\left(1-\sqrt{r}\right)^{2}}\log d\right\rceil,
	\end{equation}
	and assume that $M \geq \Meff$. 
	Then, with probability at least $1-\left( K+1\right)/ d$, the thresholding algorithm with 
	\begin{equation}\label{eq:tm_Mlarge}
	\tm=\sqrt{2 \log \left( d-K\right)}
	\end{equation}
	and $\Mc$ machines recovers the support of the $K$-sparse vector $\muv$ using  
	\begin{equation}\label{eq:th_comm_large_M}
	O\left(K \left( d-K\right) ^{\left( 1-\sqrt{r}\right) ^{2}} \log^{2.5} d\right) 
	\end{equation}
	communication bits in expectation. 
\end{thm}
\end{subtheorem} 

Let us now compare the Top-$L$ and thresholding algorithms, in terms of communication cost and recovery guarantees. 
By Theorems \ref{thm:topL} and \ref{thm:th_large_M}, with appropriately set parameters the algorithms exhibit qualitatively similar performances for high SNR and large number of machines $M$. The main differences between the two algorithms occur when $M$ is small, for example logarithmic in $d$. If the SNR is low, for example $r=O\left(\log^{-2} d \right) $, then the Top-$K$ algorithm might fail to recover the support, whereas, by Theorem \ref{thm:th_logd}, the thresholding algorithm succeeds to recover it. However, substituting $r=O\left(\log^{-2} d \right) $ in Eq. \eqref{eq:th_comm_large_M} results in total communication cost superlinear in $d$. 
In contrast, if the SNR is slightly higher, namely $r=O\left(\frac{ \log^2\log d}{\log^2 d} \right) $, then by Theorems \ref{thm:topL} and \ref{thm:th_logd}, with high probability both algorithms succeed, and the Top-$K$ algorithm incurs less total communication cost than the thresholding algorithm.
However, the thresholding algorithm is more robust in the following sense. If the sparsity level $K$ is fixed and the center only knows an upper bound on it $K_{\max}=cK$ for $c> 1$, then the Top-$K_{\max}$ algorithm incurs a communication cost that is linear in $c$, while the thresholding algorithm incurs a communication cost that is roughly the same as when $c=1$.

\subsection{Extension to sub-Gaussian noise}\label{sec:sub-G}
{Let us outline in this section how some of our results above can be extended to the case of additive sub-Gaussian noise. 
Specifically, we assume that 
each machine $i$ has $n$ i.i.d. samples of the form $\x^j_i=\muv + \errv^j_i$ for $j=1,\dots,n$, where the mean vector $\muv\in \R^d$ is exactly $K$-sparse and each noise coordinate $\err^j_{i,k}$ is an i.i.d. sub-Gaussian random variable with parameter $\vp^2$ (also known as the variance proxy).
We assume all noise coordinates have the same variance $\subvar^2=E\left[\left( \err^j_{i,k}\right) ^{2}\right]$ and finite third absolute moment $\ThM=E\left[ |\err^j_{i,k}|^{3}\right] $.
It is easy to show that $\vp^2\geq \subvar^2$ \citep[Lemma 1.4]{rigollet201518}. In our analysis, we shall assume
that for some fixed $0< \lambda\leq 1$, 
\begin{equation}\label{eq:lambda}
	\subvar^2\geq \lambda^2 \cdot \vp^2. 
\end{equation}
To account for having $n$ samples per machine, we generalize the definition of the scaling parameter $r$ as follows
\begin{equation}
	\mumin = \frac{\subvar}{\sqrt{n}}\sqrt{2 r \log \left( d-K\right) }.
	\label{eq:mu_min_r_subG}
\end{equation}}

{Denote by thresholding* a variant of the thresholding algorithm, where each contacted machine $i$ computes the following normalized empirical mean vector 
\begin{equation}
	\tilde{\x}_i=\frac{1}{\subvar\sqrt{n}}\sum_{j=1}^n \x^j_{i}. 
\end{equation}
Accordingly, each machine $i$ computes its message as   
\begin{equation}
	\y_i=\left\lbrace j\in \left[ d\right] : \tilde{x}_{i,j}>\htm \right\rbrace .
\end{equation}
Note that the effective signal strength in each machine, corresponding to its sample $\tilde{\x}_i$, is $\frac{\sqrt{n}\mumin}{\subvar}=\sqrt{2 r \log \left( d-K\right) }$, which matches Eq. \eqref{eq:mu_min_r} above. 
}

{Given sufficiently many samples per machine, results similar to those we proved for Gaussian noise hold for the case of sub-Gaussian noise. As an example, the following theorem is a variant of Theorem \ref{thm:th_large_M} for 
the thresholding algorithm.
Its proof appears in Appendix \ref{app:sub_gauss}.
A similar result can be derived for the top-$L$ algorithm.}

\begin{thm}\label{thm:subGaus}
	{Consider exact support recovery with $n$ samples per machine, corrupted by additive sub-Gaussian noise as described above. 
	Assume that $d-K$ is sufficiently large, that 
	for a suitable universal constant $C>0$
	\begin{equation}\label{eq:r_cond}
	\Omega\left(\frac{1}{\log (d-K)} \right)<\left(1-\sqrt{r}\right)^{2}< C \lambda^{2},
	\end{equation}
	and that 
	\begin{equation}
	n=\Omega\left(\frac{\ThM^2}{\subvar^6}  (1-\sqrt{r})^6\log^3 (d-K)\right) . 
	\label{eq:condition_n}
	\end{equation}
	Let $\Mc=O\left( \left( d-K\right) ^{\left( 1-\sqrt{r}\right) ^{2}} \log^{1.5} d\right)$
	and assume that $M \geq \Mc$. 
	Then, with probability at least $1-O\left( \frac{K}d\right)$, the thresholding* algorithm with 
	\begin{equation}\label{eq:tm_sub}
	\tm= \sqrt{2 \log \left( d-K\right)}
	\end{equation}
	and $\Mc$ machines recovers the support of the $K$-sparse vector $\muv$ using  
	\begin{equation}\label{eq:th_comm_sub}
	O\left(K \left( d-K\right) ^{\left( 1-\sqrt{r}\right) ^{2}} \log^{2.5} d\right) 
	\end{equation}
	communication bits in expectation.  }
\end{thm}

{The proof of Theorem \ref{thm:subGaus} uses both lower bounds and upper bounds on the tail probability of the noise. For the tail lower bound, we use a result of \citet{nagaev2002lower}, which requires a minimal number of samples per machine, as stated in 
Eq. \eqref{eq:condition_n}. Note that this requirement is rather mild. For $r$ bounded away from one, only 
a polylogarithmic in $d$ number of samples per machine suffices. 
For the lower bound to hold, we also require in \eqref{eq:r_cond} that 
the SNR parameter $r$ cannot be arbitrarily close to 1, as otherwise $n$ could tend to zero in Eq. \eqref{eq:condition_n}. In contrast, such an upper bound on $r$ does not appear in Theorem \ref{thm:th_large_M}. }

{Another key difference from Theorem \ref{thm:th_large_M} is a strict lower bound on the SNR $r$, as stated in Eq. \eqref{eq:r_cond}, 
which implies $r > \left(1-\sqrt{C}\lambda \right)^2 $. The reason for this is a rather crude upper tail probability approximation we apply in our proof, which uses the sub-Gaussian property of the noise. We remark that 
if we require a much larger number of samples per machine, then results closer to Theorem \ref{thm:th_large_M}
may be derived, even without assuming sub-Gaussianity of the noise. In particular, with sufficient number of samples per machine, the lower bound on the SNR $r$ will not depend on the parameter $\lambda$. }

\section{Sublinear distributed algorithms with small $\ell_2$ risk}\label{sec:ell_2}

In the previous section we considered distributed estimation of the support of $\muv$. 
Another common task is to estimate the vector $\muv$ itself, with both small $\ell_2$ risk
and low total communication. 
We show that this can be achieved with only a single additional round of communication. 
Furthermore, under certain parameter regimes, specifically $KM\ll \tfrac{d}{\log d}$, the resulting estimate
achieves the centralized $\ell_2$ risk, with sublinear total communication. 
The proof of this result is based on the fact that both of our algorithms achieve exact support
recovery with high probability. We thus first discuss the relation between support recovery
and $\ell_2$ risk, as well as lower bounds for the centralized minimax risk.

\subsection{On exact support recovery and $\ell_2$ risk}
Let us first briefly discuss estimation of $\muv$ in a centralized setting with $M$ samples and noise level $\sigma$.
Without any assumptions on the vector $\muv$, the empirical mean 
$\bar{\x}=\frac1M \sum_i\x_i$ is a rate-optimal estimator. When $\muv$ is assumed to be sparse, various works suggested and theoretically
analyzed the set of diagonal estimators $\Od$. An estimator $\hat{\muv}\in \Od$ has the form 
$\hat{\mu}_j=a_j\left( \bar{x}_j\right)  \bar{x}_j $ for all $j\in\left[ d\right] $, where each $a_j\left( \cdot \right) $ is a scalar function.
For further details see for example \citet[Chapter 11]{mallat1999wavelet}.

\paragraph{Projection oracle risk.}
In analyzing the lowest risk achievable in the set $\Od$, 
a key notion is the \textit{projection oracle risk}, defined as the smallest expected $\ell_2$ error of a diagonal projection estimator $\hat{\muv}^{\oracle}$ but with additional prior knowledge of $\muv$, such that $\hat{\mu}^{\oracle}_j=a_j\left( \mu_j\right)  \bar{x}_j  $ and $a_j\in\left\lbrace 0,1\right\rbrace $. It is easy to show that $\hat{\mu}^{\oracle}_j= \bar{x}_j \cdot {\bf 1}(|\mu_j| > \sigma/\sqrt{M})$. Its corresponding risk is
\begin{equation}\label{eq:R_oracle}
R_{\oracle}(\muv) = \E\left[\left\Vert \muv-\hat{\muv}^{\oracle}\right\Vert ^{2}\right] =\sum_{j=1}^d \min\left\{\frac{\sigma^2}M,\mu_j^2\right\} \leq \frac{K\sigma^2}M . 
\end{equation}
Note that the projection oracle is not a realizable estimator, as it relies on knowledge of the underlying $\muv$ for support recovery. However, the oracle risk provides a lower bound for the risk of any diagonal estimator.  
Also note that given a lower bound on the SNR, of the form 
$\min_{j\in\indset} |\mu_j| > \sigma/\sqrt{M}$, 
the oracle risk is $R_{\oracle}(\muv)  = K\sigma^2/M$.

\paragraph{Centralized lower bound.}
\citet[Theorem 3]{donoho1994ideal} proved the following lower bound on the asymptotic minimax rate among all diagonal estimators,
\begin{equation}
\lim_{d\rightarrow\infty }\inf_{\hat{\muv}\in \Od}\sup_{\muv\in \mathbb{R}^d}\frac{\mathbb{E}[\|\hat{\muv}-\muv\|^2]}{\frac{\sigma^2}M+R_{\oracle}(\muv)}\frac{1}{2 \log d} =1.
\label{eq:centralized_minimax}
\end{equation}
Moreover, they proved that thresholding at a suitable level achieves this minimax rate.

In the result above, no assumptions are made neither regarding the
sparsity of $\muv$, nor on its SNR or equivalently on $\mumin$. 
Indeed, the proof of (\ref{eq:centralized_minimax}) relies on a construction of vectors $\muv$ with $\log d$ coordinates having values slightly smaller than 
$\frac{\sigma}{\sqrt{M}}\sqrt{2\log d}$, namely with a low SNR. 
Thus, it cannot be used as a lower bound for the centralized minimax rate in our setting.
In fact, if $\muv$ is $K$-sparse and $\mumin$ is sufficiently high, then asymptotically as $d\to\infty$ with $\tfrac{KM\log d}{d}\to 0$, the risk of a suitable thresholding
estimator is equal to $R_{\oracle}(\muv)\left( 1 +o(1)\right) $. 
The reason is that in this case one can achieve exact support recovery with high probability. 
We now prove a similar result for the distributed setting.

\subsection{The $\ell_2$ risk of the Top-$L$ and thresholding algorithms}

\begin{algorithm}[t]
	\caption{Protocol $\Pi$}
	\label{alg:Pi}
	\vspace{6pt}
	\textbf{At the fusion center:}\\
	\textbf{Input} estimated support set $\hat{\indset}$\\
	\textbf{Output} setup message $\s$
	\begin{algorithmic}    
		\State {send message $\s$ which contains the set $\hat{\indset}$ to each of the $M$ machines}  \vspace{6pt}
	\end{algorithmic}
	\textbf{At each machine $i=1,\dots,M$:}\\
	\textbf{Input} setup message $\s$, sample $\x_i$, precision parameters $U,P$\\
	\textbf{Output} message $\w_i$ to center
	\begin{algorithmic}[1]
		\State  for each $k\in\hat{\indset}$, calculate $\w_{i,k}=\Trunc(x_{i,k},U,P)$
		\State	send to center $\w_i=\left\lbrace \w_{i,k}:k\in\hat{\indset}\right\rbrace $ 
		\vspace{6pt}
	\end{algorithmic}
	\textbf{At the fusion center:}\\
	\textbf{Input} messages $\w_1,\dots,\w_M$\\
	\textbf{Output} estimated vector $\hat{\muv}$
	\begin{algorithmic}[1]     
		\State  for each $i\in[M]$ and each $k\in\hat{\indset}$, reconstruct $z_{i,k}=\Approx(\w_{i,k},U,P)$
		\State	for each $k\in\hat{\indset}$, calculate the mean $\bar{z}_k=\frac{1}{M}\sum_{i\in[M]}z_{i,k}$
		\State	return $\hat{\muv}^{\Pi}$ where $\hat{\mu}^{\Pi}_j=\bar{z}_{j}\cdot \mathds{1}\left\lbrace j\in\hat{\indset} \right\rbrace$
	\end{algorithmic}
\end{algorithm} 

The Top-$L$ and thresholding algorithms described in Section \ref{sec:algs}, output an estimated support set $\hat\indset$. As we describe now, using an additional round of communication, the center can also estimate the vector $\muv$ itself. In particular, we consider the following protocol, denoted $\Pi$:  
First, the center sends the indices of $\hat\indset$ to all $M$ machines. Then, each machine $i$ replies with the binary representation $\w_{i,k}=\Trunc(x_{i,k},U,P)$ for the estimated support coordinates $k\in\hat{\indset}$, for appropriately chosen $U,P=O(\log d)$. 
The center {computes} $z_{i,k}=\Approx(\w_{i,k},U,P)$ and calculates the empirical mean $\bar{z}_k=\frac{1}{M}\sum_{i\in[M]}z_{i,k}$. Finally, the center estimates $\muv$ as follows \[\hat{\mu}^{\Pi}_j=\bar{z}_{j}\cdot \mathds{1}\left\lbrace j\in\hat{\indset} \right\rbrace.\]
The scheme is outlined in Algorithm \ref{alg:Pi}.

{The following corollary shows that applying $\Pi$ to the set $\hat{\indset}$ computed by one of our algorithms yields an estimator $\hat{\muv}^{\Pi}$ with $\ell_2$ risk $R_{\Pi}=\E\left[\left\Vert \muv-\hat{\muv}^{\Pi}\right\Vert ^{2}\right]$ which is near-oracle}. Its proof appears in Appendix \ref{app:cor_ell_2}.
\begin{corollary}\label{cor:ell_2}
	Let $d\geq 5$. Assume that the conditions of Theorem \ref{thm:topL} hold and let $\hat{\indset}\subset [d]$ be the estimate computed by the Top-$L$ algorithm. 
	In addition, assume that $\mumax<d^{\gamma}$ for $\gamma>0$. 
	Then, the $\ell_2$ risk of $\hat{\muv}^{\Pi}$ with precision parameters $P=\left\lceil\log_2 d \right\rceil$ and $U=\left\lfloor\log_2 (d^{\gamma}+\sqrt{4(\gamma+1)\log d})\right\rfloor$ is bounded as follows
	\begin{equation}\label{eq:R_Pi}
	R_{\Pi}  \leq \frac{K}{M}\left(1+d^{-1}+d^{-2}\right) +\frac{2K\mumin^2}{d}.
	\end{equation}
	The expected total communication cost of $\Pi$ is $O\left( KM\log d \right) $. 
	Thus, in an asymptotic setting where $K,M,d\to\infty$ with $\frac{KM\log d}{d} \to 0$, the protocol $\Pi$ has sublinear expected communication cost and its $\ell_2$ risk is $R_{\oracle}\left( \muv\right)\left( 1 +o(1)\right) $. 
\end{corollary}
If we assume that the conditions of either Theorem \ref{thm:th_logd}, Theorem \ref{thm:th_M} or Theorem \ref{thm:th_large_M} hold, then essentially the same proof shows that a two-round algorithm that first estimates the support of $\muv$ by the respective thresholding algorithm and then applies protocol $\Pi$ as a second round to estimate the vector $\muv$ itself can achieve near-oracle $\ell_2$ risk as well. {Similarly, the expected total communication cost is sublinear in $d$ if $KM\ll \tfrac{d}{\log d}$.}

\begin{remark}
	{An interesting question is whether one round of sublinear communication suffices to estimate $\muv$ with near-oracle $\ell_2$ risk. 
		A natural candidate solution is a variant of the thresholding algorithm where each machine sends its indices that pass the threshold $\tm$ and their corresponding coordinate values truncated to $O(\log d)$ precision.
		If the number of machines is large, then our analysis suggests that only a small fraction of the machines would send messages to the center, which would result in high risk compared to the centralized risk.
		However, if $M=O(\log d)$, then by our analysis of Theorem 2.A, at least half of the machines would send to the center each of the support elements, which should be sufficient information for estimating $\muv$ with near-centralized rate. Note that the sent coordinate values are biased, and thus simply computing their mean would result in an over-estimate of each $\mu_j$.
		Therefore, the analysis of Theorem \ref{thm:th_logd} and Corollary \ref{cor:ell_2} cannot be applied directly to this one-round variant. 
		We believe that a more delicate fusion technique should result in estimating $\muv$ with small $\ell_2$ risk, but we do not investigate this further due to our focus on support recovery.
}
\end{remark}

\section{Relation to previous works}\label{sec:prev}
In the context of the distributed sparse normal means problem, several works derived communication lower bounds for exact support recovery and for the $\ell_2$ risk of any distributed scheme with total communication budget $B$. 
We now describe in further detail three closely related previous works and their relation to our results.

\subsection{Lower bounds on the $\ell_2$ risk in distributed settings}
\citet[Theorem 4.5]{braverman2016communication} and \citet[Theorem 7]{han2018geometric} derived communication lower bounds for the distributed minimax $\ell_2$ risk of estimating a $K$-sparse vector $\muv$. Their results imply that to achieve the centralized minimax rate, the required total communication by any distributed algorithm  must be at least linear in $d$. 
However, their proof relies on sparse vectors with a very low signal-to-noise ratio.
In contrast, in scenarios where the SNR is sufficiently high these bounds do not apply, and as our theoretical analysis reveals, 
both exact support recovery and rate-optimal $\ell_2$ risk are achievable with sublinear communication, provided that $KM\ll \tfrac{d}{\log d}$. 

In more detail, 
\citet{braverman2016communication} considered blackboard communication protocols, where all machines communicate via a public blackboard and the total number of bits that they can write in the transcript is bounded by $B$. Denote the set of estimators whose inputs are blackboard communication protocols by $\Obb$ and the set of all $K$-sparse $d$ dimensional vectors as $\SdK$.
Their Theorem 4.5 states that if $d>2K$, then the $\ell_2$ risk of any distributed estimator in this model is lower bounded by
\begin{equation}
\Rbb=\inf_{\hat{\muv}\in \Obb}\sup_{\muv\in\SdK}\mathbb{E}[\|\hat{\muv}-\muv\|^2] \geq  \Omega\left(  
\min\left\lbrace \sigma^2 K,\max\left\lbrace \sigma^2 K\frac{d}B,\frac{\sigma^2 K}{M} \right\rbrace \right\rbrace \right).
\label{eq:LB_Braverman}
\end{equation}
Note that if the total communication $B$ is sublinear in $d$, then the above simplifies to $\Omega(\sigma^2 K)$, which is significantly larger 
than the centralized minimax rate, Eq. (\ref{eq:centralized_minimax}).  
The reason is that $\Rbb$ involves a supremum over all $K$-sparse vectors, without any assumptions on their SNR. 
Indeed, in their analysis a vector $\muv$ with extremely low SNR is used to prove the bound. 

\citet{han2018geometric} considered a more restricted case of one-shot protocols where each of the $M$ machines has a budget of at most $b$ bits that are sent simultaneously to the center, i.e. $B=Mb$. 
Denote by $\Osim$ the set of estimators based on such protocols. Their Theorem 7 states that if $d\geq 2K$ and $M\geq\frac{Kd^2\log\left( d/K\right)}{\min\left\lbrace b^2,d^2\right\rbrace }$, then the risk is lower bounded by
\begin{equation}
\Rsim=\inf_{\hat{\muv}\in \Osim}\sup_{\muv\in\SdK}\mathbb{E}[\|\hat{\muv}-\muv\|^2] \geq \Omega\left( \frac{\sigma^2 K}{M} \cdot \log\left( d/K\right)  \cdot  
\max\left\lbrace \frac{d}{b},1 \right\rbrace  \right).
\label{eq:LB_Han}
\end{equation}
Two remarks are in place here.
First, our protocol $\Pi$ described in Section \ref{sec:ell_2} requires two rounds of two-way communication between the center and the machines instead of one-round of one-way communication from the machines to the center. In addition, during the first round a subset of the machines may not be contacted and thus remain idle. For these reasons our estimator $\hat{\muv}^{\Pi}$ is not in $\Osim$, and thus the above lower bound does not apply to it.

Second, the lower bound \eqref{eq:LB_Han} does not apply for estimators in $\Osim$ with sublinear communication, since the condition on $M$ translates to requiring $B\geq d$. To show this, notice that if $B<d$ then in particular each machine has a sublinear communication budget, i.e., $b=d^{\beta}$ for $0<\beta<1$. The requirement on the number of machines then translates to $M\geq Kd^{2-2\beta}\log\left( d/K\right)$, and thus the total communication budget is $B=Mb\geq Kd^{2-\beta}\log\left( d/K\right)$, which is superlinear in $d$ for all $\beta<1$.

\subsection{Lower bound on exact support recovery in a distributed setting}
{\cite{shamir2014fundamental} proved  lower bounds for several distributed estimation problems under communication
constraints.
Shamir considered distributed $\left(b,n,M \right) $ protocols whereby each machine $i\in[M]$ constructs a message $\y_i$ of length at most $b$ bits based on its own $n$ i.i.d. samples and the messages $\y_1,\dots \y_{i-1}$ sent by the previous $i-1$ machines. 
Shamir considered a specific problem of distributed detection of a special coordinate $j \in[d]$, whose mean is $\tau>0$, whereas the mean of all other coordinates $i \neq j $ is zero. 
The following corollary of Shamir's Theorem 6 upper bounds the success probability of detecting $j $ by any distributed $\left(b,n,M \right) $ protocol. 
For completeness, its proof appears in the appendix.}

\begin{corollary}\label{cor:Shamir_sub}
	{Consider the class of exact support recovery problems in $d \geq 21$ dimensions, and all possible distributions of a $d$-dimensional random vector $\u$ such that: 
	\begin{enumerate}
		\item There exists one coordinate $j $ for which $\E\left[u_{j} \right]=\tau>0 $ with $\tau=O(\tfrac{1}{d\log d})$, whereas $\E\left[u_{i}\right]=0 $ for all other coordinates $i \neq j $.
		\item All coordinates $i\in[d]$ have the same second moment $\subvar^2=\E[u_i^2]=\tfrac1d$.
		\item For all coordinates $i\in[d]$, the random variable $\left( u_i-\E\left[u_i \right]\right) \sim\subG(1)$.
	\end{enumerate}
	Assume that $n \leq cd\log d$ for a suitable constant $c>0$. 
	Then for any estimate $\hat{J}$ of $j$ returned by a $\left(b, n ,M  \right) $ protocol, there exists a distribution as above such that
	\begin{equation}
		\Pr\left[\hat{J}=j\right]\leq O\left(\frac{1}{d}+\sqrt{\frac{Mb}{d}}\right).
		\label{eq:Success_Lower_Bound}
	\end{equation}}
\end{corollary}

{We now discuss the implication of this lower bound to our setting. Assume that
each of $M$ machines has $n$ i.i.d. samples of a vector $\x$ with a distribution as in Corollary \ref{cor:Shamir_sub}. Similar to \eqref{eq:mu_min_r_subG}, we define the effective SNR parameter $r$
via the relation $\tau = \frac{\tilde\sigma}{\sqrt{n}}\sqrt{2r \log d}$. 
Taking $\tau=\tfrac{C}{d\log d}$
and $n = c d \log d$ gives an effective SNR
$r=O(\tfrac1{\log^2 d})$. 
Suppose that each machine sends a message of length $b$ bits, such that the total communication is sublinear in $d$, namely 
$Mb\ll d$. Then by Corollary  \ref{cor:Shamir_sub} the probability of exact support recovery by any $(b,n,M)$ distributed scheme with $n=cd\log d$ samples per machine 
is $o\left(  1\right) $. }

{It is important to remark that the problem considered in our work and that in Corollary  \ref{cor:Shamir_sub} are somewhat different. Specifically, the distribution 
constructed to prove Corollary  \ref{cor:Shamir_sub} is not of the form signal plus noise, with the noise being independent of the signal. The setting where each sample is
of the form of a sparse signal plus additive noise is a sub-class of the distributions considered in Corollary \ref{cor:Shamir_sub}, and thus may admit lower bounds that beat Shamir's bound. 
In fact, as we prove in Section \ref{sec:algs}, for SNR parameters
only slightly higher than $O\left( \frac{1}{\log^2 d}\right) $, namely $r >\Omega \left( \frac{\log^2\log d}{\log^2 d}\right) $, exact support
recovery for signal plus Gaussian noise type observations is possible using sublinear communication. 
It would be interesting to study if any distributed
scheme can recover the support using sublinear communication for SNR values below our aforementioned bound, and to derive tight lower bounds for signal plus noise type distributions.}

\section{Simulations}\label{sec:sim}

\begin{figure}[t!]
	\centering
	\begin{subfigure}[t]{\linewidth}
		\includegraphics[width=\linewidth]{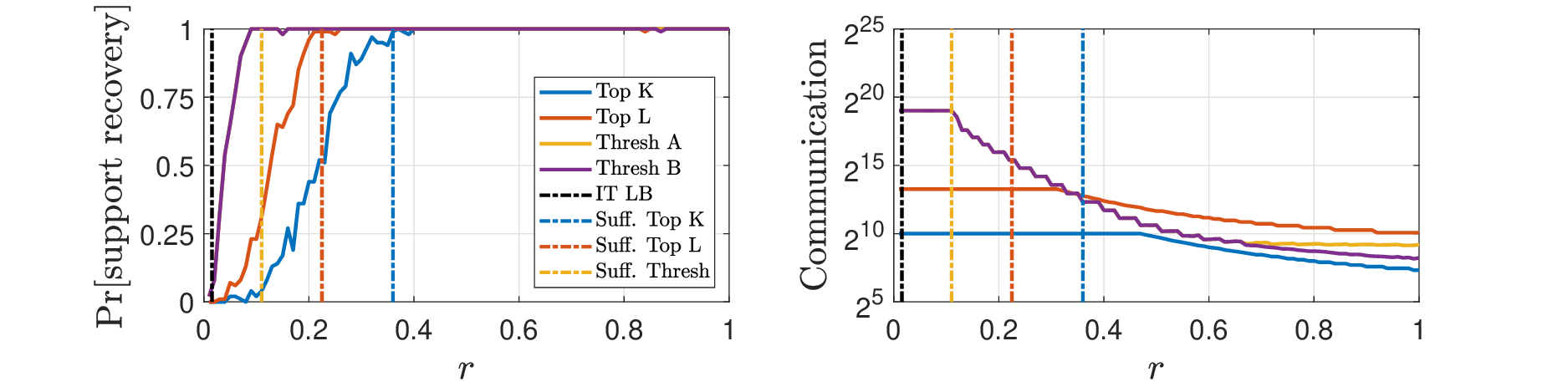}
		\caption{Setting 1: $d=2^{15}$, $M=2^6$, $K=1$, $L=10$}
	\end{subfigure}
	\begin{subfigure}[t]{\linewidth}
		\includegraphics[width=\linewidth]{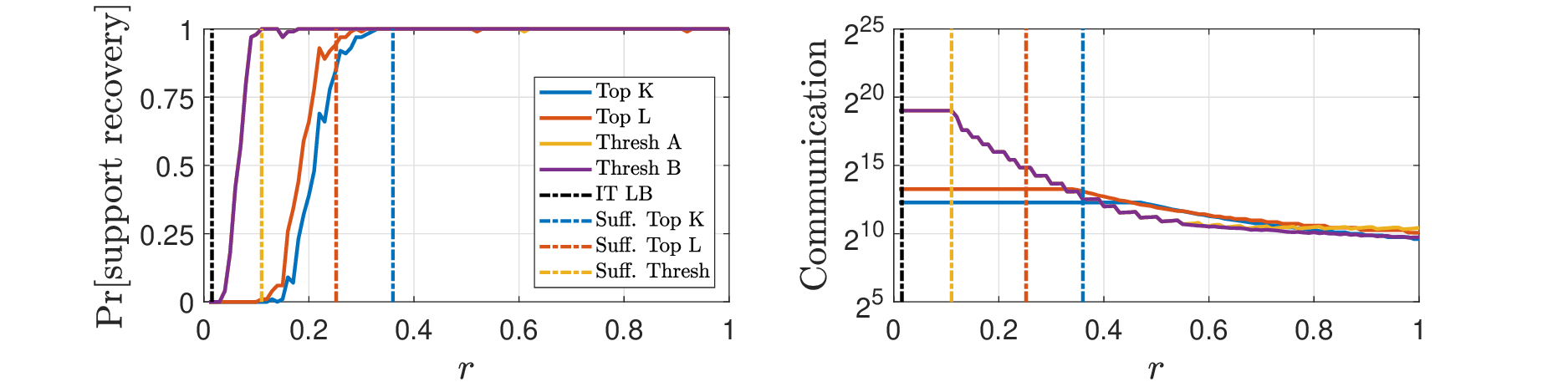}
		\caption{Setting 2: $d=2^{15}$, $M=2^6$, $K=5$, $L=10$}
	\end{subfigure}
	\begin{subfigure}[t]{\linewidth}
		\includegraphics[width=\linewidth]{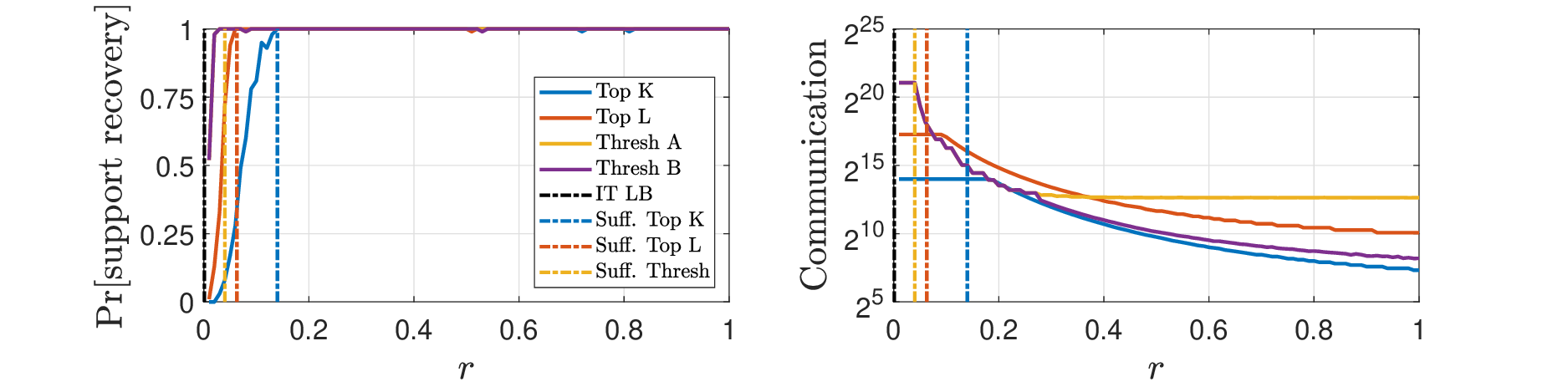}
		\caption{Setting 3: $d=2^{15}$, $M=2^{10}$, $K=1$, $L=10$}
	\end{subfigure}
	\caption{The plots on the left depict the success probability of the algorithms from Section \ref{sec:algs} as a function of $r$ in Settings 1-3. 
		The plots on the right depict the communication cost of the algorithms on a logarithmic scale as a function of $r$ in these settings.
		The blue curve corresponds to the Top-$K$ algorithm, the red curve corresponds to the Top-$L$ algorithm, the orange and purple curves correspond to variants A and B of the thresholding algorithm, respectively. 
		The vertical black line is a lower bound on the performance of all algorithms.
		The colored vertical lines are sufficient SNR bounds for the corresponding algorithms, as described in the main text.
	}
	\label{fig:r}
\end{figure}

We present several simulations that illustrate the ability of our algorithms to detect the support of a $K$-sparse $d$-dimensional vector $\muv$ with sublinear communication. We compare the performance of the Top-$L$ algorithm {with} $L=K$ (blue), the Top-$L$ algorithm {with} $L>K$ (red), variant A of the thresholding algorithm which contacts all machines, i.e., $\Mc=M$ (orange), and variant B of the thresholding algorithm which limits the number of contacted machines, i.e., $\Mc<M$ (purple). See Appendix \ref{sec:th_setting} for details on optimizing simulation parameters.

Figure \ref{fig:r} depicts the success probabilities and communication costs (on a logarithmic scale) of the aforementioned algorithms as a function of $r$, averaged over $100$ noise realizations. 
We consider three different settings of parameters $M$ and $K$. In all settings the dimension is $d=2^{15}$ and in the Top-$L$ algorithm with $L>K$ we set $L=10$. In Setting 1, $M=2^6$ and $K=1$; in Setting 2, $M=2^6$ and $K=5$; and in Setting 3 $M=2^{10}$ and $K=1$.

{The vertical black dashed line is the centralized information theoretic lower bound of $1/M$. 
This line represents the necessary SNR, below which even centralized algorithms fail with high probability. }
In addition, we define a sufficient SNR bound for each algorithm, above which it exactly recovers the support with high probability $1-O(K/d)$. 
The vertical blue and red dashed lines correspond to sufficient SNR bounds for the Top-$K$ and Top-$L$ algorithms, respectively.
The vertical orange dashed line corresponds to the sufficient SNR bounds for the thresholding algorithms.
Note that these bounds are conservative, and while they are quite tight in the presented settings, the actual range of SNRs where the algorithms are successful is often larger.

The simulation results reveal several interesting behaviors. 
First, when the SNR is extremely low, i.e., to the left of the dashed black line, none of our algorithms succeeds with high probability.
Second, no algorithm uniformly outperforms the others for all parameter regimes.
At low SNR values, the thresholding algorithms have a higher success probability compared to the Top-$L$ algorithms, but require higher communication costs. Similarly, at low SNR values the success probability of the Top-$L$ algorithm increases with $L$ at the expense of higher communication.
At high SNR values, all algorithms succeed with high probability, but the communication costs of the algorithms depend on the parameter settings. For example, the Top-$K$ algorithm can either incur a lower communication cost compared to the thresholding B algorithm (Setting 1), or a higher one (Setting 2), or they can be comparable (Setting 3).
In addition, there is a wide range of SNR values for which the communication costs of all algorithms decrease exponentially with $r$ and their total communication costs are sublinear in $d$.

To understand how a higher sparsity level $K$ affects the performance of the algorithms, we compare between Setting 1 and Setting 2. The communication cost of the Top-$K$ algorithm increases linearly with $K$. In contrast, dependence of the communication costs of the thresholding algorithms on $K$ varies with the SNR. Specifically, at low SNR values they are comparable for different values of $K$, but for high SNR values they increase linearly with $K$. This phenomenon is consistent with the higher number of messages containing support indices. 

Finally we compare between Setting 1 and Setting 3 to understand how the availability of more machines affects the performance of the algorithms.
With more machines, the Top-$L$ algorithms succeed at much lower SNR values, at the expense of higher communication costs. 
Variant A of the thresholding algorithm has a higher communication cost in Setting 3 compared to Setting 1 since it uses all machines. However, there is still a large range of SNR values where it is smaller than $d$, due to its adaptive threshold. 
As shown by our proofs, when $M$ is large, variant B of the thresholding algorithm performs similarly to the Top-$K$ algorithm, and they outperform the other algorithms.

\section*{Funding}
R.K. was partially supported by ONR Award N00014-18-1-2364, the Israel Science Foundation grant \#1086/18, and a Minerva Foundation grant. {This research was partially supported by the Israeli Council for Higher Education (CHE) via the Weizmann Data Science Research Center.}

\appendix
\section{Proofs}\label{sec:proofs}
{Denote the complement of the standard normal cumulative distribution function by $\Phi^c(t)=\Pr\left[Z>t \right] $ where $Z\sim\Nor(0,1)$.}
In our proofs we shall use the following well known auxiliary lemmas.
\begin{lemma}[Gaussian tail bounds]\label{lem:GaussianTailBounds}
        For  $t>0$,
        \begin{equation}\label{eq:GaussianTailBoundUpper}
        \frac{t}{\sqrt{2\pi}(t^2+1)}e^{-t^{2}/2} \leq \Phi^c(t)\leq\frac{1}{\sqrt{2\pi}t}e^{-t^{2}/2}. 
        \end{equation}  
        If in addition $t\geq 1$,
        \begin{equation}\label{eq:GaussianTailBoundLower}
        \Phi^c(t)\geq\frac{1}{2\sqrt{2\pi}t}e^{-t^{2}/2}.
        \end{equation}
\end{lemma}

A consequence of Eq. \eqref{eq:GaussianTailBoundUpper} is that the maximum of $n-1$ i.i.d. standard normal random variables $\Z_1,\dots,\Z_{n-1}\sim \Nor\left( 0,1\right)$ is highly concentrated around $\sqrt{2\log n}$. 
In particular, by the well known identity $ \left(1-\frac{1}{n}\right)^{n-1}\geq\frac{1}{e}$, for all $n\geq 2$
\begin{eqnarray}\label{eq:Max_Gaussians}
        \Pr\left[\max_{i\in\left[n-1\right]}\Z_{i}>\sqrt{2\log n}\right] & = & 1-\left(1-\Phi^c\left( \sqrt{2\log n}\right) \right)^{n-1} \nonumber\\
        & \leq & 1-\left(1-\frac{1}{n\sqrt{4\pi\log n}}\right)^{n-1}\leq 1-\left(1-\frac{1}{n}\right)^{n-1}\nonumber\\
        & \leq & 1-e^{-1},
\end{eqnarray}
where the third step follows from $4\pi\log n>1$.

\begin{lemma}[\citet{chernoff1952measure}]\label{lem:Chernoff}
        Suppose $X_1,\dots,X_n$ are i.i.d. Bernoulli random variables and let $X$ denote their sum. Then, for any $\delta\geq 0$,
        \begin{equation}\label{eq:ChenoffAbove}
        \Pr\left[X\geq\left(1+\delta\right)\E\left[X\right]\right]\leq e^{-\frac{\delta^{2}\E\left[X\right]}{2+\delta}}, 
        \end{equation}  
        and for any $0\leq \delta\leq 1$,
        \begin{equation}\label{eq:ChenoffBellow}
        \Pr\left[X\leq\left(1-\delta\right)\E\left[X\right]\right]\leq e^{-\frac{\delta^{2}\E\left[X\right]}{2}}. 
        \end{equation} 
\end{lemma}

{Towards proving the main theorems, we introduce a few definitions.
Denote by $I_{i,k}$ the indicator that machine $i$ sends the index $k$ to the fusion center. 
Note that for each $k$, the random variables $I_{1,k},\dots,I_{M,k}$ are independent and identically distributed.
Further denote $p_{k}=\Pr\left[I_{i,k}=1\right]$ and notice that it is the same for all machines $i$.
Our proofs use a stochastic dominance argument for lower bounding the number of votes $\v_k$ received by each support index $k\in\indset$. Towards this goal, we define a Binomial random variable $\vmin\sim Bin \left(\Meff,\pmin \right) $, where $\pmin$ is the probability that machine $i$ sends a support index whose nonzero coordinate is $\mumin$.
By definition of $\mumin$, the random variable $\vmin$ is stochastically dominated by $\v_k\sim Bin \left(\Meff,p_k \right) $ for each $k\in S$.
For exact support recovery, it suffices that for some threshold $\tc$, each support index receives {more than} $\tc$ votes, and each non-support index receives {less than} $\tc$ votes.
For {our} proof, we set the threshold 
\begin{equation}\label{eq:tc}
	\tc=4\log d.
\end{equation}
}

We conclude this subsection with two useful lemmas. 
First, we show that if $\pmin$ is sufficiently high, then any support index receives a number of votes exceeding $\tc$ with high probability.
\begin{lemma}\label{lem:supp_prob}
	{Let $\pmin$ be the probability defined above, let $\Mc$ be the number of contacted machines, and let $\tc$ be the threshold in Eq. \eqref{eq:tc}.}
	If $\pmin\geq\frac{2\tc}{\Meff}$, then 
	\begin{equation}\label{eq:supp_prob}
		\Pr[\min_{k\in\indset} \v_k <\tc] \leq \frac{K}{d}.
	\end{equation}	
\end{lemma}
\begin{proof}
	Let $\delta=1-\frac{\tc}{\Meff\pmin}$. 
	By the Chernoff bound \eqref{eq:ChenoffBellow},
	\begin{equation}\label{eq:B_bound1}
	\Pr[\vmin < \tc]  = \Pr\left[\vmin <\Meff\pmin (1-\delta)\right] \leq  \exp\left( -\Meff\pmin\delta^2/2\right).
	\end{equation}
	The assumption $\pmin\geq\frac{2\tc}{\Meff}$ implies that $\delta>1/2$ and $\delta\Meff\pmin>\tc$.
	Thus
	\begin{equation}\label{eq:B_bound2}
	\Pr[\vmin < \tc]  \leq  \exp\left( -\tc/4\right) \leq 1/d ,
	\end{equation}
	where the last inequality follows {from Eq. \eqref{eq:tc}.}
	
	Now, fix $k\in\indset$. 
	By the independence of the noises in different machines $\v_k\sim Bin \left(\Meff,p_k \right) $. 
	By definition of $\mumin$, the coordinate $\mu_k\geq\mumin$ and thus $p_k\geq\pmin$. 
	Therefore,
	\[\Pr[\v_k < \tc]\leq \Pr[\vmin < \tc].\]
	{By \eqref{eq:B_bound2}, applying a union bound over $k\in S$ proves \eqref{eq:supp_prob}.}
\end{proof}

Next, let us consider the non-support coordinates. 
The following lemma shows that if $p_j$ is sufficiently low for each non-support index $j\notin\indset$, then no non-support index receives more than $\tc$ votes with high probability.
\begin{lemma}\label{lem:non_supp_prob}
	{Let $\Mc$ be the number of contacted machines and let $\tc$ be the threshold in Eq. \eqref{eq:tc}.}
	If for each $j\notin\indset$, the probability $p_j\leq\frac{\tc}{5\Meff}$, then 
	\[
	\Pr[\max_{j\notin\indset} \v_j >\tc] \leq \frac{1}{d}.
	\]
\end{lemma}
\begin{proof}
	The average number of messages at the fusion center containing index $j$ is $\mathbb{E}[\v_j] = \Meff p_j$. 
	Let \[
	\delta  = \frac{1}{\Meff p_j}\left(\tc - \Meff p_j \right) = \frac{\tc}{\Meff p_j} - 1 .
	\]
	The assumption $p_j\leq\frac{\tc}{5\Meff}$ implies that $\delta \Meff p_j =\tc-\Meff p_j \geq \frac{4\tc}{5}$ and $\delta \geq 4$, which in turn implies that $\delta/(2+\delta)\geq 2/3$.
	Note that for each $j\in[d]$ the random variables $I_{1,j},\dots,I_{M,j}$ are independent. By a Chernoff bound \eqref{eq:ChenoffAbove},
	\begin{equation}
	\Pr[\v_{j}>\tc]=\Pr\left[\v_{j}>\Meff p_{j}(1+\delta)\right]\leq\exp\left(-\frac{\delta}{2+\delta}\delta\Meff p_{j}\right)\leq\exp\left(-\frac{8\tc}{15}\right).
	\label{eq:v_j_M4}
	\end{equation}
	{By Eq. \eqref{eq:tc},} the above probability is at most $d^{-2}$. We conclude by applying a union bound,
	$\Pr[\max_{j\not\in\indset} \v_j > \tc] \leq (d-k)\Pr[\v_j > \tc] \leq 1/d.$
\end{proof}

\subsection{Proof of Theorem \ref{thm:top}}\label{sec:proof_thm_top}
        We begin by proving Theorem \ref{thm:top1} where $L=K=1$ and then outline the necessary changes in order to prove Theorem \ref{thm:topL} for $L\geq K\geq 1$.
        
        For future use, note that by definition of the Top-$L$ algorithm, the probability that machine $i$ sends a coordinate $k\in[d]$ is 
        \begin{equation}\label{eq:pk_topL}
        p_k=\Pr \left[\exists j_1,\dots,j_{d-L}\in [d]\setminus\left\lbrace k\right\rbrace   : \,\, x_{i,k}>x_{i,j_1},\dots,x_{i,j_{d-L}} \right].
        \end{equation}

        The communication of the Top-$L$ algorithm is $B= O\left( L\Meff \log d   \right) $ since the center sends one message to each participating machine indicating $L<d$, and each of these machines sends back exactly $L$ indices. 
        \paragraph{Proof of Theorem \ref{thm:top1}.}
        Without loss of generality, let the support index be $\indset=\left\lbrace 1\right\rbrace $. Thus, 
        \[p_{1}  =  \Pr\left[x_{i,1}>\max_{j>1}x_{i,j}\right].\]
        We show that w.h.p. both $\v_{1}>\tc$ and  $\v_{j}<\tc$ for all $j>1$.
        
        By the law of total probability and the independence of the random variables $\err_{i,j}$,
        \begin{eqnarray*}
        	\pmin & = & \Pr\left[\sqrt{2r\log d}+\err_{i,1}>\max_{j>1}\err_{i,j}\right]\\
        	& \geq & \Pr\left[\sqrt{2r\log d}+\err_{i,1}>\sqrt{2\log d}\;\;|\;\max_{j>1}\err_{i,j}<\sqrt{2\log d}\right]\cdot \Pr\left[\max_{j>1}\err_{i,j}<\sqrt{2\log d}\right]\\
        	& = & \Pr\left[\err_{i,1}>\left(1-\sqrt{r}\right)\sqrt{2\log d}\right]\cdot \Pr\left[\max_{j>1}\err_{i,j}<\sqrt{2\log d}\right].
        \end{eqnarray*}
        Recall that the random variables $\err_{i,j}$ are i.i.d. standard Gaussians.
        By Eq. \eqref{eq:Max_Gaussians}, 
        \[
        \Pr\left[\max_{j>1}\err_{i,j}<\sqrt{2\log d}\right]\geq e^{-1}.
        \]
        Therefore, by the Gaussian tail bound \eqref{eq:GaussianTailBoundUpper}, 
        \begin{equation}\label{eq:p1}
        \pmin\geq e^{-1}\cdot\frac{\left(1-\sqrt{r}\right)\sqrt{2\log d}}{\sqrt{2\pi}\left(2\left(1-\sqrt{r}\right)^{2}\log d+1\right)}d^{-\left(1-\sqrt{r}\right)^{2}}.
        \end{equation}
        Combining Eq. \eqref{eq:p1} with the bound \eqref{eq:M_0} implies that $\pmin	\geq	\frac{2\tc}{\Meff}$, and thus we can apply Lemma \ref{lem:supp_prob} and get that $\Pr\left[\v_{1}<\tc\right]\leq d^{-1}.$

        Now consider a non-support index $j>1$. By symmetry considerations, the probability that machine $i$ sends $j$ to the center is \[p_j = \frac{1-p_1}{d-1}.\]
        Recall that by definition of $\mumin$, the coordinate $\mu_k\geq\mumin$ and thus $p_1\geq\pmin$. Since for any strictly positive SNR $p_1>\pmin> \frac{1}{d}$, it follows that $p_j< \frac{1}{d}$ for each $j>1$.
        Hence, the expected number of votes for index $j$ is $\E \left[ \v_{j}\right] = \Meff p_j < \frac{\Meff }{d}$. Let $\delta = \frac{\tc}{\Meff p_{j}}-1$ and note that the assumption $M\leq d$ implies that $\Meff p_j\leq 1$  and hence $\delta\geq 4\log d -1 >0$ . By the Chernoff bound \eqref{eq:ChenoffAbove}, 
        \begin{eqnarray*}
                \Pr\left[\v_{j}>\tc\right] & = & \Pr\left[\sum_{i=1}^{\Meff}I_{i,j}>\left(1+\delta\right)\Meff p_{j}\right]\\
                & \leq & e^{-\frac{\delta^{2}\Meff p_{j}}{\delta+2}}=e^{-\frac{\left(4\log d-\Meff p_{j}\right)^{2}}{4\log d+\Meff p_{j}}}\\
                & = & e^{-4\log d-\Meff p_{j}+4\Meff p_{j}\frac{4\log d}{4\log d+\Meff p_{j}}}\leq e^{3}d^{-4}.
        \end{eqnarray*}
        By a union bound over all $d-1 $ non-support coordinates, 
        \[
        \Pr\left[\max_{j>1}\v_{j}>\tc\right]\leq\left(d-1\right)\cdot e^{3}d^{-4} \leq e^{3}d^{-3}.
        \]
        By an additional union bound on the two events, the algorithm outputs the correct support index with probability at least $1-d^{-1}-e^{3}d^{-3}$.
\qed

\paragraph{Proof of Theorem \ref{thm:topL}.}
        The proof is similar to that of Theorem \ref{thm:top1}, with the following changes.
        For any threshold $a\in \R$, the probability $p_k$ that $k\in \indset$ is sent to the fusion center is lower bounded by 
        \begin{eqnarray*}
        	p_{k} & \geq & \Pr\left[x_{i,k}>a\,\,,\,\,\sum_{j\notin\indset}\mathds{1}\left\{ x_{i,j}>a\right\} \leq L-K\right]\\
        	& = & \Pr\left[\err_{i,k}>a-\mu_{k}\,\,,\,\,\sum_{j\notin\indset}\mathds{1}\left\{ \err_{i,j}>a\right\} \leq L-K\right].
        \end{eqnarray*}
        Set $a=a\left(K,L,d\right)$ and $ b=b\left(K,L,d,r \right)$ by Eqs. \eqref{eq:a} and \eqref{eq:b} respectively.
        Recall that $\err_{i,j}$ are i.i.d. for all $i\in \left[ \Meff \right] $ and $j\in \left[ d\right] $, i.e., the two events in the probability above are independent of each other. Combining this with the definition of $\mumin$ yields
        \begin{equation}\label{eq:p_k_1}
        \pmin\geq\Phi^c(b)\cdot\Pr\left[\sum_{j\notin\indset}\mathds{1}\left\{ \Z_{j}>a\right\} \leq L-K\right],
        \end{equation}
        where $\Z_j\sim\Nor\left( 0,1\right) $.
		
		We begin by bounding the first term of Eq. \eqref{eq:p_k_1}.
		If $b\leq 0$ then $\Phi^c(b)\geq 1/2$.
		Otherwise, by the Gaussian tail bound \eqref{eq:GaussianTailBoundUpper}, 
		\[\Phi^c(b)\geq\frac{b}{\sqrt{2\pi}\left(b^{2}+1\right)}\left(d-K\right)^{-\left(\sqrt{1-\frac{\log\left(L-K+1\right)}{\log\left(d-K\right)}}-\sqrt{r}\right)^{2}}.\]

		Next, we show that with probability $\geq \frac{1}{4}$ the number of non-support indices that pass the threshold $a$ is upper bounded by $L-K$. 
        Denote by $p_a$ the probability that a standard normal random variable passes the threshold $a$, i.e., $p_a\equiv \Phi^c(a)$. 
		By Eq. \eqref{eq:GaussianTailBoundUpper}, $p_a$ is upper bounded by
		\begin{equation}\label{eq:p_a}
		p_a\leq\frac{1}{\sqrt{2\pi}a}\cdot\frac{L-K+1}{d-K}.
		\end{equation}
        Next, let $\delta=\frac{L-K+1}{ p_a\left( d-K\right)}-1$. Note that the assumption $K\leq L<(d-K)/2$ implies that $\sqrt{2\pi}a\geq \sqrt{4\pi\log 2}>1$, and thus $\delta>0$.
        By the Chernoff bound \eqref{eq:ChenoffAbove}, 
        \begin{eqnarray*}
        	\Pr\left[\sum_{j\notin\indset}\mathds{1}\left\{ \Z_{j}>a\right\} \geq L-K+1\right] & = & \Pr\left[\sum_{j\notin\indset}\mathds{1}\left\{ \Z_{j}>a\right\} \geq\left(1+\delta\right)p_{a}\left(d-K\right)\right]\\
        	& \leq & e^{-\delta^{2}p_{a}\left(d-K\right)/\left(2+\delta\right)}\\
        	& = & e^{-\left(\frac{L-K+1}{p_{a}\left(d-K\right)}-1\right)^{2}\frac{p_{a}\left(d-K\right)}{1+\frac{L-K+1}{p_{a}\left(d-K\right)}}}\\
        	& = & e^{-\frac{\left(L-K+1-p_{a}\left(d-K\right)\right)^{2}}{L-K+1+p_{a}\left(d-K\right)}}.
        \end{eqnarray*}
        For $A_1,A_2>0$ the function $e^{-\frac{\left(A_1-A_2\right)^{2}}{A_1+A_2}}$ is monotonically increasing in $A_2$. Letting $A_1=L-K+1$ and $A_2=p_{a}\left(d-K\right)$, we can now apply the upper bound on $A_2$ in Eq. \eqref{eq:p_a} to the equation above. Thus the complementary probability, i.e., the second term in Eq. \eqref{eq:p_k_1}, can be lower bounded as follows
		\begin{eqnarray*}
			\Pr\left[\sum_{j\notin\indset}\mathds{1}\left\{ \Z_{j}>a\right\} \leq L-K\right] & \geq & 1-e^{-\left(L-K+1\right)\frac{\left(1-1/\sqrt{2\pi}a\right)^{2}}{1+1/\sqrt{2\pi}a}}\\
			& \geq & 1-e^{-\frac{\left(1-\sqrt{4\pi\log2}\right)^{2}}{1+\sqrt{4\pi\log2}}}\geq\frac{1}{4},
		\end{eqnarray*}
        where the second inequality follows from the assumption $K\leq L<(d-K)/2$.

        By Eq. \eqref{eq:M_K_L} the probability $\pmin \geq  \frac{2\tc}{\Meff}$ and thus $\Pr\left[\min_{k\in\indset}\v_{k}<\tc\right]\leq Kd^{-1}$ by Lemma \ref{lem:supp_prob}. 
        Let $W_i\sim Bin \left( K, \pmin   \right) $ be a binomial random variable that serves as a lower bound for how many of the support coordinates machine $i$ sends to the center.
        By the law of total probability and symmetry of the non-support indices, the probability that machine $i$ sends to the center a non-support index $j\notin \indset $ is 
        \begin{eqnarray}\label{eq:pj_topL}
                p_{j} & \leq & \sum_{n=0}^{K}\Pr\left[i\in\V_{j}|W_{i}=n\right]\cdot\Pr\left[W_{i}=n\right]\nonumber\\
                & = & \sum_{n=0}^{K}\frac{L-n}{d-K}\Pr\left[W_{i}=n\right]=\frac{L-Kp_{\min}}{d-K}\leq\frac{L}{d-K}.
        \end{eqnarray}
        Using the requirement $\MKL\leq \frac{d-K}{L}$, the rest of the proof continues in the same manner.
\qed

\subsection{Proof of Theorem \ref{thm:th}}\label{sec:proof_thm_th}
        Note that we set the precision parameters $P,U$ such that $\tm-1/d\leq\htm\leq\tm$.
        By definition of the thresholding algorithm, the probability that machine $i$ sends a support coordinate $k\in\indset$ is 
        \begin{equation}\label{eq:pk_th}
        p_k=\Pr \left[ x_{i,k}>\htm \right]\geq\Pr \left[ \err_{i,k}>\tm-\mu_k \right].
        \end{equation}
        Thus, for the extreme case $\mu_k=\mumin$,
        \begin{equation}\label{eq:pmin}
        \pmin\geq\Phi^c\left( \tm-\mumin\right). 
        \end{equation} 
        For a non-support coordinate $j\notin\indset$, the Gaussian tail bound \eqref{eq:GaussianTailBoundUpper} implies that 
	    \begin{equation}\label{eq:pj_th}
	    p_{j}=\Pr\left[\err_{i,j}>\htm\right]\leq\frac{e^{-\htm^{2}/2}}{\sqrt{2\pi}\htm}\leq e^{\tm/d}\frac{e^{-\tm^{2}/2}}{\sqrt{2\pi}(\tm-1/d)}.
	    \end{equation}
	    
	    In terms of communication, each coordinate $j\in[d]$ appears in $\Meff p_j$ messages on average. In addition, in the setup stage the fusion center sends $\Meff$ messages with the truncated threshold $\htm$, whose binary representation is $O(\log d)$ bits long. Hence the average total communication is
	    \begin{equation}\label{eq:B_th}
	    \E\left[B\right]=O\left(\Meff\log d+\left(\sum_{k\in\indset}p_{k}+\sum_{j\notin\indset}p_{j}\right)\Meff\log d\right) = O\left(\left(K+\sum_{j\notin\indset}p_{j}\right)\Meff\log d\right) ,
	    \end{equation}
		where the last step follows from the trivial bound $p_k\leq 1$ for each $k\in\indset$.
		
		We now proceed to proving the sub-theorems. 

		\paragraph{Proof of Theorem \ref{thm:th_logd}.}

		By Eqs. \eqref{eq:tm_logd} {and \eqref{eq:pmin}},
		\begin{equation}\label{eq:p1_th_logd}
		\pmin\geq\Phi^c\left( 0\right)=\frac{1}{2}.
		\end{equation}
		{Since $\Meff=\left\lceil 16\log d \right\rceil$ and by Eq. \eqref{eq:tc},} we have that $\pmin\geq 2\tc/\Meff$, and thus $\Pr[\min_{k\in\indset} \v_k <\tc] \leq \frac{K}{d}$ by Lemma \ref{lem:supp_prob}.
		Now fix $j\notin\indset$.
		Due to the assumptions $d\geq 16$ and $r>\frac{\log 5}{\log \left( d-K\right) }$, by Eq. \eqref{eq:pj_th} we have that \[p_{j}\leq\frac{e^{\sqrt{2r\log\left(d-K\right)}/d}\left(d-K\right)^{-r}}{\sqrt{2\pi}\left(\sqrt{2r\log\left(d-K\right)}-1/d\right)}\leq\frac{e^{\sqrt{2\log5}/16}}{10\left(\sqrt{\pi\log5}-1/16\right)}\leq\frac{\tc}{5\Mc}.\] 
		Applying Lemma \ref{lem:non_supp_prob} yields
		$
		\Pr[\max_{j\not\in\indset} |\v_j > \tc] \leq 1/d.
		$
		
		Finally, the average total communication follows from inserting the expressions for $p_j$ and $\Meff$ into Eq. \eqref{eq:B_th}.

		\paragraph{Proof of Theorem \ref{thm:th_M}.}
        Note that the bound $M>\sqrt{e}\cdot 32\sqrt{\pi}\log^{1.5} d$ implies that $2\log{\frac{M}{32\sqrt{\pi}\log^{1.5} d}}\geq 1$. By the expression \eqref{eq:tm} for $\tm$ and the Gaussian tail bound \eqref{eq:GaussianTailBoundLower},
        \begin{eqnarray}\label{eq:pk1}
                \pmin & \geq & \Phi^c\left( \sqrt{2\log\frac{M}{32\sqrt{\pi}\log^{1.5}d}}\right)\geq\frac{1}{2\sqrt{2\pi}\sqrt{2\log\frac{M}{32\sqrt{\pi}\log^{1.5}d}}}e^{-\log\frac{M}{32\sqrt{\pi}\log^{1.5}d}}\nonumber\\
                & = & \frac{1}{2}\cdot\frac{16\log d}{M}\sqrt{\frac{\log d}{\log\frac{M}{32\sqrt{\pi}\log^{1.5} d}}}\geq\frac{2\tc}{M},
        \end{eqnarray}
        where the last inequality follows from the upper bound on $M$.
        Thus, by Lemma \ref{lem:supp_prob}, $\min_{k\in\indset} \v_k < \tc$ with probability at most $K/d$. 
		Due to Assumption \eqref{eq:r_th_M}, $\htm\geq \sqrt{2\log\frac{5M}{\sqrt{2\pi}4\log d}}$, and thus by the first inequality of Eq. \eqref{eq:pj_th},
        \begin{equation}
        p_{j}\leq\frac{1}{\sqrt{2\pi}\sqrt{2\log\frac{5M}{\sqrt{2\pi}4\log d}}}e^{-\log\frac{5M}{\sqrt{2\pi}4\log d}}=\frac{4\log d}{5M\sqrt{2\log\frac{5M}{\sqrt{2\pi}4\log d}}}<\frac{\tc}{5M},
        \end{equation}
        where the last inequality follows from {Eq. \eqref{eq:tc}} and the condition on $M$.
        Thus by Lemma \ref{lem:non_supp_prob}, $\Pr[\max_{j\not\in\indset} \v_j > \tc]  \leq 1/d.$   
        
        Towards computing the expected communication of the algorithm, we bound $p_j$ more carefully using the second inequality of Eq. \eqref{eq:pj_th},
        \begin{eqnarray}
        	p_{j} & \leq & \frac{e^{\left(\sqrt{2r\log\left(d-K\right)}+\sqrt{2\log\frac{M}{32\sqrt{\pi}\log^{1.5}d}}\right)/d}\left(d-K\right)^{-r}e^{-2\sqrt{r\log\left(d-K\right)\log\frac{M}{32\sqrt{\pi}\log^{1.5}d}}}}{\sqrt{2}\left(\sqrt{2r\log\left(d-K\right)}+\sqrt{2\log\frac{M}{32\sqrt{\pi}\log^{1.5}d}}-\frac{1}{d}\right)}\cdot\frac{32\log^{1.5}d}{M}\nonumber\\
        	& \leq & \left(d-K\right)^{-r}e^{-2\sqrt{r\log\left(d-K\right)\log\frac{M}{32\sqrt{\pi}\log^{1.5}d}}}\cdot\frac{32\log^{1.5}d}{M},\label{eq:pj_th2}
        \end{eqnarray}
        where the second inequality follows from bounding $\sqrt{2r\log\left(d-K\right)}+\sqrt{2\log\frac{M}{32\sqrt{\pi}\log^{1.5}d}}-\frac{1}{d}>1$ and from $d\geq 15$, which implies that $e^{\left(\sqrt{2r\log\left(d-K\right)}+\sqrt{2\log\frac{M}{32\sqrt{\pi}\log^{1.5}d}}\right)/d}<\sqrt{2}$.
        By inserting Eq. \eqref{eq:pj_th2} into Eq. \eqref{eq:B_th}, the expected communication of the algorithm is 
        \[\E\left[B\right]=O\left(KM\log d+\left(d-K\right)\cdot{\left(d-K\right)^{-r}e^{-2\sqrt{r\log\left(d-K\right)\log\frac{M}{32\sqrt{\pi}\log^{1.5}d}}}}\frac{\log^{1.5}d}{M}\cdot M\log d\right).\]
        Rearranging completes the proof.
		\qed

	\paragraph{Proof of Theorem \ref{thm:th_large_M}.}
	Recall that $\Mc$ and $\tm$ are given by Eqs. \eqref{eq:Mlarge} and \eqref{eq:tm_Mlarge}, respectively.
	By the Gaussian tail bound \eqref{eq:GaussianTailBoundUpper},
	\begin{eqnarray}\label{eq:pk1_largeM}
	\pmin & \geq & \Phi^c\left(\left(1-\sqrt{r}\right)\sqrt{2\log\left(d-K\right)}\right)\nonumber\\
	& \geq & \frac{\left(1-\sqrt{r}\right)\sqrt{2\log\left(d-K\right)}}{\sqrt{2\pi}\left(\left(1-\sqrt{r}\right)^{2}2\log\left(d-K\right)+1\right)}\left(d-K\right)^{-\left(1-\sqrt{r}\right)^{2}}=\frac{2\tc}{\Meff}.
	\end{eqnarray}
	Thus by Lemma \ref{lem:supp_prob}, $\min_{k\in\indset} \v_k < \tc$ with probability at most $K/d$.

	Fix a non-support index $j\notin \indset$. 
	Note that the assumption $d-K\geq 20$ implies that $e^{\sqrt{2\log\left(d-K\right)}/d}<\sqrt{2}$. Thus, by Eq. \eqref{eq:pj_th},
	\begin{equation}\label{eq:pj_Mlarge}
	p_{j}\leq\frac{e^{\sqrt{2\log\left(d-K\right)}/d}}{\sqrt{2\pi}\left(\sqrt{2\log\left(d-K\right)}-1/d\right)}\left(d-K\right)^{-1}\leq\frac{1}{\sqrt{\pi}\left(\sqrt{2\log\left(d-K\right)}-1/d\right)}\left(d-K\right)^{-1}.
	\end{equation}
	It is easy to verify that the aforementioned assumption and the condition $r> \left( \frac{\log 10}{ \log \left( d-K\right) }\right) ^2$, or equivalently, $d^{2\sqrt{r}}\geq 100$, imply that 
	$p_{j}<\frac{\tc}{5\Meff}.$
	Thus the desired bound 
	$
	\Pr[\max_{j\not\in\indset} \v_j > \tc]  \leq 1/d
	$ 
	follows from Lemma \ref{lem:non_supp_prob}.
	
	By inserting Eq. \eqref{eq:pj_Mlarge} into Eq. \eqref{eq:B_th}, the expected communication of the algorithm is 
	\[\E\left[B\right]=O\left(K\Meff \log d+\left(d-K\right)\cdot\frac{1}{\sqrt{2\pi}\sqrt{2\log\left(d-K\right)}}\left(d-K\right)^{-1}\cdot \Meff\log d\right).\]
	Using Eq. \eqref{eq:Mlarge} concludes the proof. 
	\qed

\subsection{Proof of Theorem \ref{thm:subGaus}}\label{app:sub_gauss}
{Towards proving Theorem \ref{thm:subGaus}, we first recall the definition of sub-Gaussian random variables and cite a couple of useful results.}
\begin{definition}[{\citep[Definition 1.2]{rigollet201518}}]
	{A random variable $X\in\R$ is said to be sub-Gaussian with sub-Gaussian parameter (variance proxy) $\vp^2>0$ if $\E\left[ X\right] =0$ and its moment generating function satisfies $\E\left[ \exp(sX)\right] \leq \exp\left(\vp^2s^2/2 \right) $ for all $s\in\R$. We write $X\sim\subG(\vp^2)$.}
\end{definition}

{Let $X_1,\dots,X_n\sim\subG\left( \vp^2\right) $ be i.i.d. sub-Gaussian random variables. We denote the variance of each r.v. by $\subvar^{2}=E\left[X_{j}^{2}\right]$ and their third absolute moment by $\ThM=E|X_{j}|^{3}$, and assume that it is finite $\ThM<\infty$. 
In our proofs for the Gaussian noise case, we used the upper and lower bounds on tail probabilities described in Lemma \ref{lem:GaussianTailBounds}.
We now show that similar bounds hold for sub-Gaussian noise. }

{To prove tail lower bounds, we need the following result.}
\begin{thm}[{\cite[Corollary 3]{nagaev2002lower}}]
	\label{thm:Nagaev}
	{Let $X_1,\dots,X_n\sim\subG\left( \vp^2\right) $ be i.i.d. sub-Gaussian random variables as defined above and let $\tilde{X}=\frac{1}{\subvar\sqrt{n}}\sum_{i=1}^n X_i$.
	If $1.7<t\leq \frac{1}{25}\sqrt{n}\frac{\subvar^3}{\ThM}$, then
	\begin{equation}\label{eq:Nagaev}
	\Pr\left[ \tilde{X} > t  \right] > \Phi^c(t)\exp\left(-\frac{(2.35\ThM /\subvar^3+0.2)t^3}{\sqrt{n}} \right) \left(1- \frac{(16.88\ThM /\subvar^3+6.58)t}{\sqrt{n}}\right).  
	\end{equation}}
\end{thm}

{To prove tail upper bounds, we use the following lemma, which is an easy corollary of Lemma 1.3 and Corollary 1.7 of \citet{rigollet201518}.}
\begin{lemma}\label{lem:tildeX}
	{Let $X_1,\dots, X_n \sim\subG\left( \vp^2\right)$ be i.i.d. sub-Gaussian random variables as defined above and let $\tilde{X}=\frac{1}{\subvar\sqrt{n}}\sum_{i=1}^n X_i$. Then, $\tilde{X}\sim \subG(\vp^2/\subvar^2)$, and there exists a constant $c>0$ such that for any $t>0$,
	\begin{equation}\label{eq:sub_upp_tail}
	\Pr \left[ \tilde{X}>t \right] \leq \exp \left( -c\subvar^2t^2/\vp^2\right) . 	
	\end{equation}
	}
\end{lemma}
{We now proceed to prove the theorem.}

\paragraph{Proof of Theorem \ref{thm:subGaus}.}
{The proof is similar to that of Theorem \ref{thm:th_large_M}, with the following changes, pertaining to the probabilities of sending the support and non-support indices, i.e., Eqs. \eqref{eq:pk1_largeM} and \eqref{eq:pj_Mlarge}.}
{Let $t=(1-\sqrt{r})\sqrt{2\log (d-K)}$ and let the number of samples in each machine satisfy Eq. \eqref{eq:condition_n} as follows,
\begin{equation}\label{eq:n}
n\geq  \frac{\left(2.35\ThM /\subvar^3+0.2 \right)^2 }{\log^2 2}t^6=\frac{8\left(2.35\ThM /\subvar^3+0.2 \right)^2 }{\log^2 2}(1-\sqrt{r})^6\log^3 (d-K).
\end{equation}
Further assume that $1.7<(1-\sqrt{r})\sqrt{2\log (d-K)}\leq \frac{1}{25}\sqrt{n}\frac{\subvar^3}{\ThM}$. It is easy to verify that for the values of $n$ in Eq. \eqref{eq:n}, the SNR $r$ satisfies the left inequality in \eqref{eq:r_cond}.
We begin with applying Theorem \ref{thm:Nagaev} to lower bound the probability that machine $i$ sends a support index $k\in\indset$.
The first term of Eq. \eqref{eq:Nagaev}, i.e. $\Phi^c(t)$, is identical to that of Eq. \eqref{eq:pk1_largeM}.
By Eq. \eqref{eq:n}, the second term is at least $1/2$. By Jensen's inequality, $\ThM /\subvar^3\geq 1$. Thus, the third term is 
\[
1-(16.88\ThM /\subvar^3+6.58)\frac{t}{\sqrt{n}}> 1-\frac{4\log 2}{(1-\sqrt{r})^2\log (d-K)} ,
\]
which is also greater than $ \frac{1}{2} $ for sufficiently large $d-K$.
Hence, Theorem \ref{thm:Nagaev} implies that 
\begin{equation}
\Pr\left[ \tilde{\err}_{i,k} > (1-\sqrt{r})\sqrt{2\log (d-K)}  \right] > \frac{1}{4}\Phi^c\left( (1-\sqrt{r})\sqrt{2\log (d-K)}\right).
\end{equation}
Let 
\begin{equation}\label{eq:M_sub}
\Mc=\left\lceil  4\cdot\frac{ 8\sqrt{2\pi}\left(\left(1-\sqrt{r}\right)^{2}2\log\left(d-K\right)+1\right)}{\left(1-\sqrt{r}\right)\sqrt{2\log\left(d-K\right)}}\left(d-K\right)^{\left(1-\sqrt{r}\right)^{2}}\log d\right\rceil.
\end{equation}
Then, Eq. \eqref{eq:pk1_largeM}
can be replaced with
\begin{equation}\label{eq:pmin_sub}
\pmin > \frac{1}{4}\cdot \frac{\left(1-\sqrt{r}\right)\sqrt{2\log\left(d-K\right)}}{\sqrt{2\pi}\left(\left(1-\sqrt{r}\right)^{2}2\log\left(d-K\right)+1\right)}\left(d-K\right)^{-\left(1-\sqrt{r}\right)^{2}} =\frac{2\tc}{\Mc}.
\end{equation}
The bound $\Pr \left[  \min_{k\in\indset} \v_k < \tc\right] \leq K/d$ follows from Lemma \ref{lem:supp_prob}.}

{As for the non-support indices, by Lemma \ref{lem:tildeX} the r.v. $\tilde{\err}_{i,j} \sim \subG(\vp^2/\subvar^2)$. Thus, by the tail bound \eqref{eq:sub_upp_tail} and by Condition \eqref{eq:lambda},
\begin{equation*}
p_j = \Pr \left[\tilde{\err}_{i,j}>\htm \right] \leq \exp\left(-c\subvar^2\htm^2/\vp^2 \right) \leq \exp\left(-c\lambda^2\htm^2 \right).
\end{equation*}
As in the proof of Theorem \ref{thm:th_large_M}, recall that the truncated threshold satisfies that 
$\htm>\tm-1/d$. Hence, 
\begin{equation*}
p_j \leq \exp\left(-c\lambda^2\tm^2 \right) \exp\left(2c\lambda^2\tm/d \right). 
\end{equation*}
Next, we insert into the right hand side above the value of $\tm$, Eq. \eqref{eq:tm_sub}. 
For sufficiently large $d-K$, the second term above is bounded by say 2. Hence,  
\begin{equation}\label{eq:pj_sub}
p_j \leq O\left( (d-K)^{-2c\lambda^2}\right)  . 
\end{equation}
Condition \eqref{eq:r_cond} with a constant $C>2c$ implies that for sufficiently large $d-K$, 
as in the original proof, $p_j\leq \frac{\tc}{5\Mc}$ where $t_c$ is given
in Eq. \eqref{eq:tc}.
Thus the desired bound 
$
\Pr[\max_{j\not\in\indset} \v_j > \tc]  \leq 1/d
$ 
follows from Lemma \ref{lem:non_supp_prob}.
The communication bound \eqref{eq:th_comm_sub} follows directly from Eqs. \eqref{eq:pj_sub} and \eqref{eq:M_sub}.}

\subsection{Proof of Corollary \ref{cor:ell_2}}\label{app:cor_ell_2}

We first analyze the total communication cost of $\Pi$. 
Each machine $i$ sends a message $\w_i$ consisting of the truncated binary representations of $x_{i,k}$ for $k\in\hat{\indset}$. Recall that the length of each $\w_{i,k}$ is $P+U+2$ bits. Since $P,U=O(\log d)$, the expected total communication cost of $\Pi$ is $O\left( KM\log d \right) $.

Let $\hat{\muv}$ be the output of protocol $\Pi$, and recall that  $\hat{\mu}_j=\bar{z}_{j}\cdot \mathds{1}\left\lbrace j\in\hat{\indset} \right\rbrace $.
By linearity of expectation and the law of total probability, 
\begin{equation}\label{eq:cor1}
\E\left[\left\Vert \muv-\hat{\muv}\right\Vert _{2}^{2}\right]  = \sum_{j\in\left[d\right]}\E\left[\left(\mu_{j}-\hat{\mu}_{j}\right)^{2}\right]=\sum_{j\in\left[d\right]}\left(\E\left[\left(\mu_{j}-\bar{z}_{j}\right)^{2}\right]\Pr\left[j\in\hat{\indset}\right]+\mu_{j}^{2}\Pr\left[j\notin\hat{\indset}\right]\right).
\end{equation}
We now bound each of the terms in the RHS.

Fix $j\in[d]$. Since $\E\left[\bar{x}_{j}\right]=\mu_{j}$, it follows that
\[
\E\left[\left(\mu_{j}-\bar{z}_{j}\right)^{2}\right]=\E\left[\left(\mu_{j}-\bar{x}_{j}+\bar{x}_{j}-\bar{z}_{j}\right)^{2}\right]=\E\left[\left(\mu_{j}-\bar{x}_{j}\right)^{2}\right]+\E\left[\left(\bar{x}_{j}-\bar{z}_{j}\right)^{2}\right].
\]
Furthermore, since the noise in different machines is i.i.d.,
\[
\E\left[\left(\mu_{j}-\bar{x}_{j}\right)^{2}\right]=\frac{1}{M}\E\left[\left(\mu_{j}-x_{i,j}\right)^{2}\right]=\frac{1}{M},
\]
and 
\[
\E\left[\left(\bar{x}_{j}-\bar{z}_{j}\right)^{2}\right]=\frac{1}{M}\E\left[\left(x_{i,j}-z_{i,j}\right)^{2}\right]
\]
for any fixed $i\in[M]$.

We now bound $\E\left[\left(x_{i,j}-z_{i,j}\right)^{2}\right]$ for any fixed $j\in[d]$ and $i\in[M]$.
Since $x_{i,j}\sim \Nor (\mu_j,1)$ and $z_{i,j}$ is a deterministic function of it $z_{i,j}(x_{i,j})$, then
\[\E\left[\left(x_{i,j}-z_{i,j}\right)^{2}\right]  =  \int_{-\infty}^{\infty}\left(x-z\left(x\right)\right)^{2}\frac{\exp\left(-\left(x-\mu_{j}\right)^{2}/2\right)}{\sqrt{2\pi}}dx.\]
If $x< 2^{U+1}$, then the truncation step of the protocol implies that the remainder is bounded such that $|x-z|\leq 2^{-P}.$ Otherwise, the value $x$ is higher than the range that is representable using $U+1$ bits before the binary dot, and thus the magnitude $|x-z|$ can be as large as $|x|$ itself. 
Therefore,
\[
\E\left[\left(x_{i,j}-z_{i,j}\right)^{2}\right]\leq2^{-P}+2\int_{2^{U+1}}^{\infty}x^{2}\frac{\exp\left(-\left(x-\mu_{j}\right)^{2}/2\right)}{\sqrt{2\pi}}dx.
\]
Using integration by parts,
\[
\int_{2^{U+1}}^{\infty}x^{2}\frac{e^{-\left(x-\mu_{j}\right)^{2}/2}}{\sqrt{2\pi}}dx\leq\frac{1}{\sqrt{2\pi}}\left(2^{U+1}+\mu_{j}\right)e^{-\left(2^{U+1}-\mu_{j}\right)^{2}/2}+\left(1+\mu_{j}^{2}\right)\int_{2^{U+1}}^{\infty}\frac{e^{-\left(x-\mu_{j}\right)^{2}/2}}{\sqrt{2\pi}}dx.
\]
By the Gaussian tail bound \ref{eq:GaussianTailBoundUpper},
\begin{eqnarray*}
	\int_{2^{U+1}}^{\infty}x^{2}\frac{e^{-\left(x-\mu_{j}\right)^{2}/2}}{\sqrt{2\pi}}dx & \leq & \frac{1}{\sqrt{2\pi}}\left(2^{U+1}+\mu_{j}+\frac{1+\mu_{j}^{2}}{2^{U+1}-\mu_{j}}\right)e^{-\left(2^{U+1}-\mu_{j}\right)^{2}/2}\\
	& \leq & \frac{1}{\sqrt{2\pi}}\left(\sqrt{4\left(\gamma+1\right)\log d}+2d^{\gamma}+\frac{1+d^{2\gamma}}{\sqrt{4\left(\gamma+1\right)\log d}}\right)d^{-2\left(\gamma+1\right)}\leq\frac{1}{2d},
\end{eqnarray*}
where the second inequality follows from the bound $\mumax<d^{\gamma}$ and the selection $U$, and the last inequality holds for all $d\geq 5$ and $\gamma\geq 0$.
Finally, Since $P=\left\lceil \log_2 d \right\rceil$,
\[
\E\left[\left(x_{i,j}-z_{i,j}\right)^{2}\right]\leq d^{-2}+d^{-1}.\]
In addition, since $\E\left[ |\hat{\indset}|\right] =K$, the sum $\sum_{j\in\left[d\right]}\Pr\left[j\in\hat{\indset}\right]=K,$
and thus the first term in the RHS of Eq. \eqref{eq:cor1} is bounded by $\frac{K}{M}\left(1+ d^{-2}+d^{-1}\right) $.

It remains to prove that for each support index $k\in\indset$, $$\mu_{k}^{2}\Pr\left[k\notin\hat{\indset}\right]\leq 2\mumin^2/d.$$

	Denote by $G$ the "good" event that each non-support index $j\notin\indset$ receives less than $\tc=4\log d$ votes. 
	Fix $k\in\indset$. 
	By the law of total probability,
	\[
	\Pr\left[k\notin\hat{\indset}\right]\leq\left(\Pr\left[k\notin\hat{\indset}|G\right]\Pr\left[G\right]+\left(1-\Pr\left[G\right]\right)\right).
	\]
	Conditioned on $G$, the index $k\in \hat{\indset}$ if $ \v_k > \tc$. The complementary probability can be bounded by Chernoff \eqref{eq:ChenoffBellow},
	\[
	\Pr\left[k\notin\hat{\indset}|G\right]\leq\Pr\left[\v_{k}<\tc\right]\leq e^{-\frac{1}{2}\left(\Meff p_{k}-\tc\right)\left(1-\frac{\tc}{\Meff p_{k}}\right)}.
	\]
	Recall that under the conditions of Theorem \ref{thm:topL}, $p_j\leq \frac{\tc}{5\Meff}$ for each non-support index $j\notin\indset$. Therefore $\Pr\left[ G\right] \geq 1-1/d$ by Lemma \ref{lem:non_supp_prob}.	
	Thus,
	\begin{equation}\label{eq:func_to_min}
	\mu_{k}^{2}\Pr\left[k\notin\hat{\indset}\right]\leq\mu_{k}^{2}\left(e^{-\frac{1}{2}\left(\Meff p_{k}-\tc\right)\left(1-\frac{\tc}{\Meff p_{k}}\right)}\left(1-1/d\right)+1/d\right).
	\end{equation}
	In addition, recall that $p_k$ is defined by Eq. \eqref{eq:pk_topL} for the Top-$L$ algorithm (or by Eq. \eqref{eq:pk_th} for the thresholding algorithm), and decays exponentially with $\mu_k$. 
	Therefore, the right hand side of Eq. \eqref{eq:func_to_min} is monotonically {decreasing} in $\mu_k$, and thus upper bounded by 
	\[
	\mumin^{2}\left(e^{-\frac{1}{2}\left(\Meff \pmin-\tc\right)\left(1-\frac{\tc}{\Meff\pmin}\right)}\left(1-1/d\right)+1/d\right).
	\] 
	Recall that the assumption $\pmin\geq\frac{2\tc}{\Meff}$ also holds under the conditions of Theorem \ref{thm:topL}. Thus we can apply Eq. \eqref{eq:B_bound2} and get the desired bound
	\[
	\mu_{k}^{2}\Pr\left[k\notin\hat{\indset}\right]\leq\mumin^{2}\left(\frac{1}{d}\left(1-1/d\right)+1/d\right)\leq\frac{2\mumin^{2}}{d}.
	\]

We now turn to proving the last part of the corollary. 
The lower bound in Condition \eqref{eq:SNR_regime} implies that $\mumin>1/\sqrt{M}$. Thus, by Eq. \eqref{eq:R_oracle}, the oracle risk is $R_{\oracle}\left( \muv\right) =K/M$.
In addition, the upper bound in Condition \eqref{eq:SNR_regime} implies that $\mumin^2\leq 2\log d$. 
Taking $K,M,d\to\infty$ with $\frac{KM\log d}{d}\to 0$ yields the desired result.

\subsection{Proof of Corollary \ref{cor:Shamir_sub}}\label{app:shamir}
{The proof is similar to that of Theorem 5 of \cite{shamir2014fundamental}. The main step of the proof, which is proved in Shamir's Theorem 6, is deriving an upper bound on the probability of detecting the special coordinate $j$ for some "hard" distribution. It remains to prove that this distribution indeed satisfies the conditions specified in Corollary \ref{cor:Shamir_sub}.}

{	Within the proof of his Theorem 5, \cite{shamir2014fundamental} defined the following problem, which he referred to as hide-and-seek 2.}
	\begin{definition}[Hide-and-seek Problem 2]
		{Let $0<\rho<\frac{1}{2}$. Consider the set of distributions $\left\lbrace \Pr_j (\cdot)\right\rbrace_{j=1}^d $ over $\left\lbrace -\bm{e}_i,+\bm{e}_i\right\rbrace_{i=1}^d$, defined as
		\[
		\Pr_{j}\left(\bm{e}_{i}\right)=\begin{cases} 
			\frac{1}{2d} & i\neq j\\
			\frac{1}{2d}+\frac{\rho}{d} & i=j
		\end{cases}\quad\quad\Pr_{j}\left(-\bm{e}_{i}\right)=\begin{cases}
			\frac{1}{2d} & i\neq j\\
			\frac{1}{2d}-\frac{\rho}{d} & i=j
		\end{cases}.
		\]
		Given an i.i.d. sample of $Mn$ instances generated from $\Pr_j (\cdot)$, where $j$ is unknown, detect $j$.}
	\end{definition}
	{Let $\u\in\R^d$ be a random vector sampled from $\Pr_j (\cdot)$. 
	We now verify that it satisfies the conditions specified in Corollary \ref{cor:Shamir_sub}. 
	\begin{enumerate}
		\item By construction, there exists $j$ for which $\E[u_j]=2\rho/d$, whereas $\E[u_i]=0$ for all other coordinates $i\neq j$. Thus, the first condition holds for $\tau=2\rho/d>0$.
		\item For each coordinate $i\in[d]$, the value $u_i^2=1$ with probability $1/d$ and $0$ otherwise, and thus $\E[u_i^2]=1/d$.
		\item For each coordinate $i\in[d]$, the random variable $\left( u_i-\E\left[u_i \right]\right) $ equals $+1$ w.p. $1/2d$, $-1$ w.p. $1/2d$, and $0$ w.p. $1-1/d$. Its absolute values are bounded by $1$, and thus it is sub-Gaussian with parameter $1$.
	\end{enumerate}
	Shamir proved the following theorem, which bounds the success probability of detecting $j$.}
	\begin{thm}[{\cite[Theorem 6]{shamir2014fundamental}}]
		{Consider the hide-and-seek problem 2 on $d > 1$ coordinates, with some bias $\rho\leq \min \left\lbrace \frac{1}{27},\frac{1}{9\log d}, \frac{d}{14n} \right\rbrace $ and sample size $Mn$. Then for any estimate $\hat{J}$ of the biased coordinate returned by any $(b,n,M)$ protocol, there exists some coordinate $j$ such that
		\begin{equation}
			\Pr_j\left[\hat{J}=j\right]\leq \frac{3}{d}+11\sqrt{\frac{Mb}{d}}.
		\end{equation}}
	\end{thm}
	{To complete the proof of Corollary \ref{cor:Shamir_sub}, note that if $d\geq 21$ and $n\leq \frac{9}{14}d\log d$, then $\frac{1}{9\log d} \leq \min\left\lbrace \frac{1}{27},\frac{d}{14n} \right\rbrace $. 
	Thus, Shamir's Theorem 6 with $\tau=\frac{2\rho}{d}\leq \frac{2}{9d\log d}$ proves the corollary. }
	
\qed

\section{Simulation parameter settings} \label{sec:th_setting}
For simplicity of the proofs we did not fully optimize the choices of $\Meff$ and thresholds. We outline below the choices used for our simulations in Section \ref{sec:sim}. In terms of setup message length, in all simulations $L$ is represented by $\log L$ bits and $\htm$ is represented with $U=2$ bits before the binary dot and $P=3$ bits after the binary dot.

\paragraph{Top-$L$ algorithm.}
We define the following random variables that represent bounds on the number of votes that a support coordinate $k\in\indset$ receives $Y^{top}_s\left( d,r,K,L\right)=Bin\left(\Meff,p^{top}_s \right)  $ and on the number of votes that a non-support coordinate $j\notin\indset$ receives $Y^{top}_n\left( d,r,K,L\right)=Bin\left(\Meff,p^{top}_n \right)  $, where $p^{top}_s=p_s\left( d,r,K,L\right) $ is the probability that $k\in\indset$ is sent by machine $i$, defined in Eq. \eqref{eq:p_k_1}, and $p^{top}_n=p_n\left( d,r,K,L\right) $ is the probability that $j\notin\indset$ is sent by machine $i$, defined in Eq. \eqref{eq:pj_topL}. 

With high probability $Y^{th}_n$ does not deviate from its expectation by more than a $\frac{\log (d-K)}{\log \log (d-K)} $ multiplicative bound.
Thus, we set the number of contacted machines as \[\Meff=\max\left\lbrace \left\lceil \frac{1}{p^{top}_{s}\left(d,r,K,L\right)}\cdot \frac{\log(d-K)}{\log\log(d-K)}\right\rceil ,1\right\rbrace .\]
Intuitively, this selection ensures that the expected number of votes for a fixed support index is equal to the maximal expected number of votes for any non-support index.

In all of our simulations $\frac{\log (d-K)}{\log \log (d-K)}\cdot \E Y^{top}_n  <1$. Hence, the sufficient SNR bound for the Top-$L$ algorithm (vertical blue/red line) is the minimal $r$ for which $\E  Y^{top}_s    \geq 2$, i.e., the support indices have at least 2 votes in expectation while the non-support indices have at most 1.

\paragraph{Thresholding algorithm.}
Similarly to the calculation for the Top-$L$ algorithm, we define $Y^{th}_s\left( d,r,K,\tm\right)=Bin\left(\Meff,p^{th}_s \right)$ and $Y^{th}_n\left( d,r,K,\tm\right)=Bin\left(\Meff,p^{th}_n \right)$ as the number of votes for a support coordinate and non-support coordinate respectively, where $p^{th}_s$ is by Eq. \eqref{eq:pmin} and $p^{th}_n$ is by Eq. \eqref{eq:pj_th}. 

For variant A, given $r$ and $M$, we set the number of contacted machines $\Meff=M$ and the threshold $\tm $ as the highest $t$ s.t.
\begin{equation}\label{eq:emp_th}
\Pr\left[Y^{th}_s\left( d,r,K,t\right)<\E Y^{th}_{n}\left( d,r,K,t\right)\frac{\log\left(d-K\right)}{\log\log\left(d-K\right)}\right]<\frac 1 d.
\end{equation}
Intuitively, Eq. \eqref{eq:emp_th} requires that the probability that the number of votes for a fixed support index is higher than the maximal expected number of votes for any non-support index is lower than $d^{-1}$.

In variant B, the parameters $\tm$ and $\Meff$ are set in the following manner. If for $\Meff=M$ the threshold $t<\sqrt{2\log \frac{d-K}{K}}$, then $\tm=t$ as in variant A. Otherwise, we set $\tm=\sqrt{2\log \frac{d-K}{K}}$ and take the lowest $\Meff$ for which Eq. \eqref{eq:emp_th} with $t=\tm$ holds.

Let $r_{\min}$ and $t_{\min}$ denote the minimal $r$ value and the corresponding $t$ value for which Eq. \eqref{eq:emp_th} holds, respectively.
$r_{\min}$ is the sufficient SNR bound for the thresholding algorithm (vertical orange line).
Note that when $r<r_{\min}$, there is no value of $t$ for which this Eq. \eqref{eq:emp_th} holds.
For completeness of the simulations, in this case we set $t=t_{\min}$.

\bibliography{sparse_bib}
\bibliographystyle{plainnat}

\end{document}